\documentclass{jpc} 

\usepackage{mark}

\usepackage{nicefrac}       
\usepackage{algorithmic}
\usepackage{booktabs}       
\usepackage{multirow}
\usepackage{amsmath,mathtools}


\usepackage{natbib}
\usepackage[ruled]{algorithm2e}
\theoremstyle{plain} 

\begin{document}

\title{Private Boosted Decision Trees via Smooth Re-Weighting}

\author[V.~Asadi]{Vahid R. Asadi}	
\address{Department of Computer Science, Simon Fraser University, Burnaby, BC, Canada}	
\email{vasadi@sfu.ca}  

\author[M.~Carmosino]{Marco L. Carmosino}	
\address{Department of Computer Science, Boston University, Boston, MA, United States}	
\email{marco@ntime.org}  

\author[M.~Jahanara]{Mohammadmahdi Jahanara}	
\address{Department of Computer Science, Simon Fraser University, Burnaby, BC, Canada}	
\email{mjahanar@sfu.ca}  

\author[A.~Rafiey]{Akbar Rafiey}	
\address{Department of Computer Science, Simon Fraser University, Burnaby, BC, Canada}	
\email{arafiey@sfu.ca}  

\author[B.~Salamatian]{Bahar Salamatian}	
\address{Department of Computer Science, Simon Fraser University, Burnaby, BC, Canada}	
\email{bsalamat@sfu.ca}  





\begin{abstract}
  \noindent Protecting the privacy of people whose data is used by machine learning algorithms is important. Differential Privacy is the appropriate mathematical framework for formal guarantees of privacy, and boosted decision trees are a popular machine learning technique. So we propose and test a practical algorithm for boosting decision trees that guarantees differential privacy. Privacy is enforced because our booster never puts too much weight on any one example; this ensures that each individual's data never influences a single tree "too much." Experiments show that this boosting algorithm can produce better model sparsity and accuracy than other differentially private ensemble classifiers.
\end{abstract}

\maketitle
\section{Introduction}
\label{sec:introduction}

Boosted decision trees are a popular, widely deployed, and successful
machine learning technique. Boosting constructs an ensemble of
decision trees sequentially, by calling a decision tree \emph{base
  learner} with sample weights that ``concentrate attention'' on
training examples that are poorly classified by trees constructed so
far \cite{10.5555/2207821}.

\emph{Differential Privacy} (DP) is a mathematical definition of
privacy which ensures that the distribution over hypotheses produced
by a learning algorithm does not depend ``too much'' (quantified by
$\epsilon$) on any one input example
\cite{DBLP:conf/icalp/Dwork06}. An adversary cannot even tell if a
specific individual participated in a differentially private study or
not \citep[see][section ~IV.C.1]{630158}.

Recent purely theoretical work used \emph{Smooth Boosting} ---
algorithms that never concentrate too much sample weight on any one
example --- to give a simple and differentially private algorithm for
learning large-margin half-spaces \cite{DBLP:conf/colt/BunCS20}. Their
boosting algorithm is generic; it does not depend on any specific
features of the weak learner beyond differential privacy.

\subsection{Our contributions}
Here, we demonstrate that the smooth boosting algorithm of
\citet{DBLP:conf/colt/BunCS20} is a practical and efficient
differentially private classifier when paired with decision ``stumps''
--- depth-1 trees. We compare on three classification tasks to DP
logistic regression \cite{DBLP:journals/jmlr/ChaudhuriMS11}, DP
bagging \cite{DP-Bagging-JordonYS19}, DP gradient boosting
\cite{DBLP:conf/aaai/LiWWH20}, and smooth boosting over our own
``reference implementation'' of DP decision trees. In all cases,
smooth-boosted decision \emph{stumps} improve on other algorithms in
accuracy, model sparsity, or both in the high-privacy regime. This is
surprising; in the non-private setting somewhat deeper trees (depth 3
- 7) generally improve accuracy. It seems that stumps better tolerate
the amount of noise that must be added to enforce privacy for small
samples. Since many applications of DP (e.g., US Census sample
surveys, genetic data) require simple and accurate models for small
datasets, we regard the high utility of smooth-boosted DP-Stumps in
these settings as a \emph{pleasant} surprise. In order to analyze the privacy of our algorithm, we introduce the
notion of Weighted Exponential Mechanism and Weighted Return Noisy Max Mechanism based on
the novel notion of robust sensitivity which we believe is of independent interest.

\subsection{Related Work}
\label{sec:related-work}
Decision trees are one of the most popular classifiers; often used for
their efficiency and interpretability. Since the NP-completeness
result of \citet{NP-HyafilR76}, there has been an
extensive body of research devoted to designing heuristic algorithms
for inducing decision trees. These algorithms are efficient and
successful in practice \cite{doi:10.1142/9097}. Notable examples are
greedy procedures such as ID3, C4.5, and CART
\cite{DBLP:journals/ml/Quinlan86, DBLP:books/mk/Quinlan93,
  DBLP:books/wa/BreimanFOS84}. They iteratively ``grow'' a single tree
by adding children to some leaf node of an existing tree according to
a \emph{splitting criterion}.

\paragraph{\textbf{Differentially Private Decision Trees.}}
\label{sec:diff-priv-decis}
Many previous works explored differentially private algorithms for
learning \emph{single} decision trees. Authors in \citet{BlumDMN05}
showed how a traditional non-private algorithm (ID3) could be modified
to achieve differential privacy by adding noise to the splitting
criterion. \citet{FriedmanS10} empirically demonstrated the
effectiveness of using the exponential mechanism to privately select
splits for ID3 and C4.5.

Recent work modified the TopDown algorithm of Kearns and Mansour
\cite{KearnsM96} to enforce differential privacy \cite{Kaiwen}. This
is particularly interesting because TopDown is \emph{not} a
heuristic. Under a \emph{weak learning} assumption --- if the features
considered for splitting have some advantage over random guessing ---
TopDown is guaranteed to learn a tree with low training error. Wang,
Dick, and Balcan~\cite{Kaiwen} preserve this guarantee under
differential privacy by appealing to the utility of the Laplace
Mechanism. Here, we implement a simpler DP-TopDown algorithm --- as the
goal of our work is to test differentially private \emph{boosting,}
weaker tree induction is perferable.

\paragraph{\textbf{Differentially Private Boosting.}}
\label{sec:diff-priv-boost}

Differentially private boosting is less well-studied because the
iterative structure of boosting algorithms complicates the task of
enforcing privacy while maintaining utility. In theory,
\cite{dwork2010boosting} designed the first differentially private
boosting algorithm. Later, Bun, Carmosino, and
Sorrell~\cite{DBLP:conf/colt/BunCS20} offered a much simpler private
algorithm based on the hard-core lemma of \cite{BarakHK09}. Both
algorithms preserved privacy by using ``smooth'' distributions over
the sample to limit the ``attention'' any one example receives from a
base learner. Our \texttt{LazyBB} (Algorithm \ref{alg:LB-NxM}) is an
implementation of the algorithm of \cite{DBLP:conf/colt/BunCS20} over
decison trees and stumps.

\emph{Boosting by reweighting} updates an explicit distributions over the data, where the probability mass on an example reflects how difficult it is to classify. \emph{Gradient Boosting} iteratively fits the residuals of the combined voting classifier --- it alters the labels instead of explicit weights on each sample.

One very recent experimental work studies differentially private
\emph{gradient} tree boosting \cite{DBLP:conf/aaai/LiWWH20}. 
Their base learner is an ensemble of greedily-constructed
decision trees on disjoint subsets of the data, so that parallel
composition may be used \emph{inside} the base learner to save
privacy. They deal with the ``too much attention'' problem by
\emph{clipping} the pseudo-residuals at each round, so that outliers
do not compromise privacy by over-influencing the hypothesis at any
round. They use composition to spread the privacy budget across each
round of boosting.

Our algorithm is boosting by \emph{reweighting} and uses much
simpler base learners. Our update rule is just multiplicative weights,
and we enforce privacy by \emph{projecting} the resulting distribution
over examples into the space of smooth distributions. Our algorithm
remains accurate in the \emph{high-privacy} ($\epsilon < 1$) setting;
\cite{DBLP:conf/aaai/LiWWH20} did not explore this regime.

\section{Preliminaries}

\subsection{Distributions and Smoothness}
To preserve privacy, we will never concentrate too much ``attention''
on a single example. This can be enforced by only using \emph{smooth
  distributions} --- where no example is allowed to have too much
relative weight.

\begin{defi}[$\kappa$-Smooth Distributions]
  A probability distribution $\mu$ on domain $X$ is
  \emph{$\kappa$-smooth} if for each $x \in X$ we have
  $\mu(x) \leq \frac{1}{\kappa |X|}$, where $\kappa \in [0,1]$.
\end{defi}

To maintain the invariant that we only call base learners on smooth
distributions, we Bregman-project onto the space of \emph{high density
  measures}. High density measures\footnote{A \emph{measure} is a
  function from the domain to $[0,1]$ that need not sum to one;
  normalizing measures to total weight naturally results in a
  distribution.} correspond to smooth probability
distributions. Indeed, the measure $\mu$ over $X$ has density at least
$\kappa$ if and only if the probability distribution
$\frac{1}{|\mu|} \mu$ satisfies smoothness
$\mu(x) \leq \frac{|\mu|}{\kappa |X|}$ for all $x \in X$ where
$|\mu|=\sum_{x\in X}\mu(x)$ and density of $\mu$ is $|\mu|/|X|$.


\begin{defi}[Bregman Projection]
  Let $\Gamma \subseteq \R^{|S|}$ be a non-empty closed convex set of
  measures over $S$. The \emph{Bregman projection} of $\tilde{\mu}$
  onto $\Gamma$ is defined as:
  $
    \Pi_\Gamma \tilde{\mu} =
    \arg\min_{\mu \in \Gamma} \KL{\mu}{\tilde{\mu}}.
  $
\end{defi}

The result of Bregman 1967 says Bregman projections do not badly
``distort'' KL-divergence. Moreover, when $\Gamma$ is the set of
$\kappa$-dense measures we can compute $\Pi_\Gamma \tilde{\mu}$ for
measure $\tilde{\mu}$ with $|\unconstmu | < \kappa |X|$
\cite{BarakHK09}. Finally, we require the following notion of
similarity.

\begin{defi}[Statistical Distance]
  The \emph{statistical distance}, a.k.a. \emph{total variation
    distance}, between two distributions $\mu$ and $\nu$ on $\Omega$,
  denoted $\mathtt{d}(\mu, \nu)$, is defined as
  $
    \mathtt{d}(\mu,\nu) = \max_{S\subset \Omega} | \mu(S) -
    \nu(S)|.
  $
\end{defi}
For finite sets $\Omega$,
$\mathtt{d}(\mu,\nu) = \nicefrac{1}{2}\sum_{x\in \Omega}
|\mu(x)-\nu(x)|$ e.g., see Proposition 4.2 in \cite{levin2017markov}.


\subsection{Learning}
Throughout the paper we let $\mathcal{S} = \{(\mb{x}_i,y_i)\}^n$ where
$\mb{x}_i=(x_{i1},\dots , x_{ir})$ and $y_{i} \in \{+1,-1\}$ denote a
dataset where \emph{all} features and labels are Boolean. Though our
techniques readily extend to continuous-feature or multi-label
learning, studying this restricted classification setting simplifies
the presentation and experiments for this short paper.

\begin{defi}[Weak Learner]\label{def:weaklearn}
  Let $S \subset (\cX \times \{\pm 1\})^n$ be a training set of size
  $n$. Let $\mu$ be a distribution over $[n]$. A \emph{weak learning
    algorithm} with \emph{advantage} $\gamma$ takes $(S,\mu)$ as input
  and outputs a function $h \from \cX \to \{ \pm 1\}$ such that:
  $
    \Pr_{x \sim \mu}[h(x) = c(x)] \geq 1/2 + \gamma
  $
\end{defi}

\begin{defi}[Margin]
  For binary classification, the \emph{margin} (denoted $\sigma$) of
  an ensemble $H = h_{1}, \dots, h_{\tau}$ consisting of $\tau$
  hypotheses on an example $(x, y)$ is a number between $-\tau$ and
  $\tau$ that captures how ``right'' the classifier as a whole is
  $
    \sigma(H, x, y) = y\sum_{j=1}^{\tau}h(x).
  $
\end{defi}

\subsection{Differential Privacy}
The definition of differential privacy relies on the notion of
neighboring datasets. We say two datasets are neighboring if they
differ in a single record. We write $D \sim D'$ when two datasets $D$,
$D'$ are neighboring.

\begin{defi}[$(\epsilon, \delta)$-Differential Privacy
  \cite{dwork2006our}] For $\epsilon,\delta \in \R_+$, we say that a
  randomized computation $M$ is
  \emph{$(\epsilon,\delta)$-differentially private} if for any
  neighboring datasets $D \sim D'$, and for any set of outcomes
  $S \subseteq \mathrm{range}(M)$,
  \[
    \Pr[M(D)\in S] \leq \exp(\epsilon) \Pr[M(D')\in S]+\delta.
  \]
  When $\delta=0$, we say $M$ is \emph{$\epsilon$-differentially private}.
\end{defi}

Differentially private algorithms must be calibrated to the
sensitivity of the function of interest with respect to small changes
in the input dataset, defined formally as follows.

\begin{defi}[Sensitivity]
  The sensitivity of a function $F\colon X \to \R^k$ is
   $
    \max_{D\sim D' \in X}  ||F(D)-F(D')||_1.
  $
  A function with sensitivity $\Delta$ is called $\Delta$-sensitive.
\end{defi}



Two privacy composition theorems, namely sequential composition and
parallel composition, are widely used in the design of mechanisms.
\begin{thm}[Sequential Composition \cite{bun2016concentrated,
    dwork2009differential, dwork2010boosting, mcsherry2007mechanism}]
  \label{thm:Sequential-Composition} Suppose a set of privacy
  mechanisms $M = \{M_1,\dots,M_k \}$ are sequentially performed on a
  dataset, and each $M_i$ is $(\epsilon_i,\delta_i)$-differentially
  private with $\epsilon_i\leq \epsilon_0$ and
  $\delta_i \leq \delta_0$ for every $1\leq i\leq k$. Then mechanism
  $M$ satisfies $(\epsilon,\delta)$-differential privacy where
  \begin{itemize}
  \item $\epsilon=k\epsilon_0$ and $\delta=k\delta_0$ (the basic
    composition), or
  \item
    $\epsilon =
    \sqrt{2k\ln{1/\delta'}}\epsilon_0+k\epsilon_0(\mr{e}^{\epsilon_0}-1)$
    and $\delta =\delta' + k\delta_0$ for any $\delta' > 0$ (the
    advanced composition).
  \end{itemize}
\end{thm}


\begin{thm}[Parallel Composition~\cite{McSherry10}]
  \label{thm:parallel-composition}
  Let $D_1,\dots,D_k$ be a partition of the input domain and suppose
  $M_1,\dots, M_k$ are mechanisms so that $M_i$ satisfies
  $\epsilon_i$-differential privacy. Then the mechanism
  $M(S) = (M_1(S \cap D_1), \dots , M_k(S \cap D_k))$ satisfies
  $(\max_i \epsilon_i)$-differential privacy.
\end{thm}

\subsection{Differentially Private Learning}
\label{sec:DP-Learning}
Given two neighboring datasets and \emph{almost} the same
distributions on them, privacy requires weak learners to output the same
hypothesis with high probability. This idea was formalized for 
zero-concentrated differential privacy (zCDP) in Definition 18 of \cite{DBLP:conf/colt/BunCS20}.
Below, we adapt it for the $(\epsilon,\delta)$-DP setting.

\begin{defi}[DP Weak Learning]
  \label{def:ed-private-wkl}
  A weak learning algorithm $\mathtt{WkL} : S \times \cD(S) \to \cH$
  is $(\epsilon, \delta, \zeta)$-differentially private if for all
  neighboring samples $S \sim S' \in (\cX^n \times \{ \pm 1\})$ and
  all $H\subseteq \cH$, and any pair of distributions
  $\dst{\mu}, \dst{\mu}'$ on $[n]$ with
  $\mathtt{d}(\dst{\mu},\dst{\mu}') < \zeta$, we have:
  \[
    \Pr[\mathtt{WKL}(S,\dst{\mu}) \in H ] \le
    \exp(\epsilon) \Pr[\mathtt{WKL}(S',\dst{\mu}') \in H] + \delta.
  \]
\end{defi}

Note that the notion of sensitivity for differentially private weak
learners depends on the promised total variation distance
$\zeta$. Hence, differentially private weak learners must be
calibrated to the sensitivity of the function of interest with respect
to small changes in the distribution on the dataset. For this purpose, we introduce
\emph{robust sensitivity} below. There is no analog of
robust sensitivity in the zCDP setting of \citet{DBLP:conf/colt/BunCS20}, 
because their private weak learner for halfspaces did not require it --- they 
exploited inherent ``compatibility'' between Gaussian noise added to preserve
privacy and the geometry of large-margin halfspaces. We do not have this luxury in
the $(\epsilon,\delta)$-DP setting, and so must reason directly about how
the accuracy of each potential weak learner changes with the distribution over examples. 

\begin{defi}[Robust Sensitivity]
  The robust sensitivity of a function
  $F\colon (X,\mathsf{M}) \to \R^k$ where $\mathsf{M}$ is the set of
  all distributions on $X$ is defined as
  $$
    \max\limits_{\substack{D\sim D' \in X\\ \dst{\mu},\dst{\mu}'\in \mathsf{M}: \mathtt{d}(\dst{\mu},\dst{\mu}') < \zeta}}  ||F(D,\dst{\mu}(D))-F(D',\dst{\mu}'(D'))||_1.
  $$
  A function with robust sensitivity $\Delta_\zeta$ is called
  $\Delta_\zeta$ robustly sensitive.
\end{defi}

The standard Exponential Mechanism \cite{mcsherry2007mechanism} does
not consider utility functions with an auxiliary weighting $\mu$. But
for weak learning we only demand privacy (close output distributions)
when \emph{both} the dataset and measures are ``close.'' When both
promises hold and $\mu$ is fixed, the Exponential Mechanism is indeed
a differentially private weak learner; see the Appendix for a proof.

\begin{defi}[Weighted Exponential Mechanism]
  \label{def:WEM}
  Let $\eta > 0$ and let $q_{D,\mu} \colon \mathcal{H}\to \R$ be a
  quality score.  Then, the \emph{Weighted Exponential Mechanism}
  ${WEM}({\eta,q_{D,\mu}})$ outputs $h \in \mathcal{H}$ with probability
  proportional to $ \exp\left(\eta\cdot q_{D,\mu}(h)\right).$
\end{defi}

Similar to the Exponential Mechanism one can prove  privacy and utility guarantee for the Weighted Exponential Mechanism.
\begin{thm}
  \label{thm:weighted-EM}
  Suppose the quality score $q_{D,\mu} \colon \mathcal{H}\to \R$ has
  robust sensitivity $\Delta_\zeta$. Then, ${WEM}({\eta,q_{D,\mu}})$
  is $(2\eta\Delta_\zeta,0,\zeta)$-differentially private weak
  learner. Moreover, for every $\beta \in (0,1)$, ${WEM}({\eta,q_{D,\mu}})$
  outputs $h\in \mathcal{H}$ so that
  \[
    \Pr\left[q_{D,\mu}(h) \geq \max_{h' \in \mathcal{H}} q_{D,\mu}(h')-\ln\left({|\mathcal{H}|}/{\beta}\right)/\eta\right] \geq 1-\beta.
  \]
\end{thm}


Another differentially private mechanism that we use is Weighted
Return Noisy Max (WRNM). Let $f_1,\dots,f_k$ be $k$ quality functions
where each $f_i:S\times \cD(S)\to \R$ maps datasets and distributions
over them to real numbers. For a dataset $S$ and distribution $\mu$
over $S$, WRNM adds independently generated Laplace noise
$Lap(1/\eta)$ to each $f_i$ and returns the index of the largest noisy
function i.e. $i^*=\argmax\limits_{i} (f_i+Z_i)$ where each $Z_i$
denotes a random variable drawn independently from the Laplace
distribution with scale parameter $1/\eta$.

\begin{thm}
  \label{thm:weighted-RNM}
  Suppose each $f_i$ has robust sensitivity at most $\Delta_\zeta$. Then
  WRNM is a $(2\eta\Delta_\zeta,0,\zeta)$-differentially private weak
  learner.
\end{thm}

\section{Private Boosting}
\label{sec:PrivateBoosting}
Our boosting algorithm, Algorithm~\ref{alg:LB-NxM}, simply calculates the current margin of each
example at each round, exponentially weights the sample accordingly,
and then calls a private base learner with smoothed sample weights.
The hypothesis returned by this base learner is added to the ensemble
$H$, then the process repeats. Privacy follows from (advanced)
composition and the definitions of differentially private weak
learning. Utility (low training error) follows from regret bounds for
lazy projected mirror descent and a reduction of boosting to zero-sum
games. Theorem \ref{thm:WBregBoost} formalizes these guarantees; for
the proof, see \cite{DBLP:conf/colt/BunCS20}. Next, we discuss the
role of each parameter.

\textbf{Round Count $\tau$.~} The number of base hypotheses.  In the
non-private setting, $\tau$ is like a regularization parameter --- we
increase it until just before overfitting is observed.  In the private
setting, there is an additional trade-off: more rounds \emph{could}
decrease training error until the amount of noise we must inject into
the weak learner at each round (to preserve privacy) overwhelms
progress.

\textbf{Learning rate $\lambda$.~} Exponential weighting is
attenuated by a \emph{learning rate} $\lambda$ to ensure that weights
do not shift too dramatically between calls to the base
learner. $\lambda$ appears negatively because the margin is negative
when the ensemble is incorrect. Signs cancel to make the weight on an
example \emph{larger} when the committee is bad, as desired.

\textbf{Smoothness $\kappa$.~} Base learners attempt to maximize their
probability of correctness over each intermediate
distribution. Suppose the $t$-th distribution was a point mass on
example $x_i$ --- this would pose a serious threat to privacy, as
hypothesis $h_t$ would only contain information about individual
$x_i$! We ensure this never happens by invoking the weak learner only
over $\kappa$-smooth distributions: each example has probability mass
``capped'' at $\frac{1}{\kappa n}$. For larger samples, we have
smaller mass caps, and so can inject less noise to enforce privacy.
Note that by setting $\kappa = 1$, we force each intermediate 
distribution to be uniform, which would entirely negate the effects
of boosting: reweighting would simply be impossible.
Conversely, taking $\kappa \rightarrow 0$ will entirely remove the smoothness constraint.

\begin{algorithm}[tb]
  \caption{\texttt{LazyBB}: Weighted Lazy-Bregman Boosting}
  \label{alg:LB-NxM}
  \emph{\textbf{Parameters:}} $\kappa \in (0,1)$,
  desired training error; $\lambda \in (0,1)$, learning rate; 
  $\tau \in \mathbb{N}$ number of rounds \\
  \emph{\textbf{Input:}} $S \in X^n$, the sample; 
  \begin{algorithmic}
    \STATE $H\gets \emptyset$ and $\mu_1(i) \gets \kappa ~~ \forall i \in [n]$
    \COMMENT{Uniform bounded measure}
    \FOR{$t = 1$ to $\tau$}
    \STATE{$\hat{\mu}_t \gets $} Normalize $\mu_t$ to a distribution
    \COMMENT{Obtaining a $\kappa$-smooth distribution}
    \STATE{$h_t \gets \mathtt{WkL}(S,\hat{\mu}_t)$}
    \STATE $H\gets H\cup\{h_t\}$
    \STATE $\sigma_{t}(i) \gets y_{i} \sum_{j=1}^{t}h_{j}(x_{i})
    ~~ \forall i \in [n]$
    \COMMENT{Normalized score of current majority vote}
    \STATE $\tilde{\mu}_{t+1}(i) \gets
    \exp\left( -\lambda \sigma_{t}(i) \right) \kappa
    ~~ \forall i \in [n]$
    \STATE $\mu_{t+1} \gets \Pi_{\Gamma}(\tilde{\mu}_{t+1})$
    \COMMENT{Bregman project to a $\kappa$-dense measure}
    \ENDFOR
    \STATE {\bfseries Output:} {$\hat{f}(x) = \operatorname{Maj}_{h_{j} \in H}\left[ h_{j}(x)\right]$}
  \end{algorithmic}
\end{algorithm}

\begin{thm}[Privacy \& Utility of \texttt{LazyBB}]
  \label{thm:WBregBoost}
  Let $L$ be a $(\epsilon_b, \delta_b, (1/\kappa n))$-DP weak learner
  with advantage $\gamma$ and failure probability $\beta$ for concept
  class $\cH$.  Running \texttt{LazyBB} with $L$ for
  $\tau \geq \frac{16\log{(1/\kappa)}}{\gamma^2}$ rounds on a sample
  of size $n$ with $\lambda = \gamma/4$ guarantees:
  \begin{description}
  \item[Privacy] \texttt{LazyBB} is
    $(\epsilon_{A}, \delta_{A})$-DP, where
$\epsilon_{A} = \sqrt{2 \tau\cdot \ln(1 / \delta')} \cdot \epsilon_{b}
      + \tau \cdot \epsilon_{b}\cdot (\exp(\epsilon_{b}) - 1)$ and $\delta_{A} = \tau\cdot\delta_{b} + \delta'$ for every $\delta' > 0$ (using advanced composition).


  \item[Utility] With all but $(\tau \cdot \beta)$ probability, $H$ has
    at least $\gamma$-good normalized margin on a $(1-\kappa)$
    fraction of $S$ i.e., 
    $
      \Pr_{(x,y) \sim S} \Big[
        \nicefrac{y}{\tau}\sum_{j = 1}^{\tau} h_{j}(x) \leq \gamma
        \Big] \leq \kappa.
    $
  \end{description}  
\end{thm}

Weak Learner failure probability $\beta$ is critical to admit because
whatever ``noise'' process a DP weak learner uses to ensure privacy
may ruin utility on some round. So, we must
union bound over this event in the training error guarantee.

\section{Concrete Private Boosting}
\label{sec:concrete-private-boosting}

Here we specify concrete weak learners and give privacy guarantees for
\texttt{LazyBB} combined with these weak learners.

\subsection{Baseline: 1-Rules}
\label{sec:baseline:-1-rules}

To establish a baseline for performance of both private and
non-private learning, we use the simplest possible hypothesis class:
1-Rules or ``Decision Stumps'' \cite{Decision-Stump-L92,Decision-Stump-Holte93}. In the Boolean feature and
classification setting, these are just constants or signed literals
(e.g. $-x_{17}$) over the data domain. 
\begin{align*}
  &\mathsf{1R}(\mathcal{S}) = \{x_{i}\}_{i \in [d]} \cup
  \{-x_{i}\}_{i \in [d]} \cup \{+1, -1\} \quad\text{and}\quad
 \\ &\mr{err}(\mc{S}, \mu, h) = \sum\limits_{(\mb{x}_i,y_i)\in \mc{S}}
  \mu(i)\chi\{ h(\mb{x}_i) \neq y\}.
\end{align*}
To learn a 1-Rule given a distribution over the training set, return
the signed feature or constant with minimum weighted error. Naturally,
we use the Weighted Exponential Mechanism with noise rate $\eta$ to privatize
selection. This is simply the Generic Private Agnostic Learner of  \cite{DBLP:journals/siamcomp/KasiviswanathanLNRS11}, finessing 
the issue that ``weighted error'' is actually a \emph{set} of
utility functions (analysis in Appendix C). We denote the baseline and differentially private
versions of this algorithm as \texttt{1R} and \texttt{DP-1R},
respectively.

\begin{thm}
  \texttt{DP-1R} is a $(4 \eta \zeta, 0, \zeta)$-DP weak learner.
\end{thm}

Given a total privacy budget of $\epsilon$, we divide it uniformly
across rounds of boosting. Then, by Theorem \ref{thm:WBregBoost}, we
solve $\epsilon = 4\tau\cdot \eta \cdot\zeta$ for $\eta$ to determine how much noise
\texttt{DP-1R} must inject at each round. Note that privacy
depends on the statistical distance $\zeta$ between distributions over
neighboring datasets. \texttt{LazyBB} furnishes the promise
that $\zeta \leq 1/\kappa n$. It is natural for $\zeta$ to depend on the
number of samples: the larger the dataset, the easier it is to
``hide'' dependence on a single individual, and the less noise we can
inject at each round. Overall:

\begin{thm}
  \texttt{LazyBB} runs for $\tau$ rounds using \texttt{DP-1R} at
  noise rate $\eta = \frac{\epsilon \kappa n}{4\tau}$ is $\epsilon$-DP.
\end{thm}

If a weak learning assumption holds --- which for 1-Rules simplifies
to ``over every smooth distribution, at least one literal or constant
has $\gamma$-advantage over random guessing'' --- then we will boost
to a ``good'' margin. We can compute the advantage of \texttt{DP-1R}
given this assumption.

\begin{thm}
\label{thm:1R-advantage}
    Under a weak learning assumption with advantage  $\gamma$, \texttt{DP-1R}, with probability at least $1-\beta$, is a weak learner with advantage at least $\gamma -\frac{1}{\eta}\ln{\frac{|\mc{H}|}{\beta}}$. That is, for any distribution $\mu$ over $\mc{S}$, we have  
    \[
        \sum\limits_{(\mb{x}_i,y_i)\in \mc{S}} \mu(i) \chi\{ h_{out}(\mb{x}_i) \neq y_i\}  \leq \nicefrac{1}{2} -\big(\gamma - \nicefrac{1}{\eta}\ln{\nicefrac{|\mc{H}|}{\beta}}\big)
    \]
    where $h_{out}$ is the output hypothesis of \texttt{DP-1R}.
\end{thm}
\subsection{TopDown Decision Trees}
\label{sec:decision-trees}

TopDown heuristics are a family of decision tree learning algorithms
that are employed by widely used software packages such as C4.5, CART,
and scikit-learn. We present a differentially private TopDown
algorithm that is a modification of decision tree learning algorithms
given by Kearns and Mansour \cite{KearnsM96}. At a high level, TopDown induces decision
trees by repeatedly \emph{splitting} a leaf node in the tree built so
far. On each iteration, the algorithm \emph{greedily} finds the leaf
and splitting function that maximally reduces an upper bound on the
error of the tree. The selected leaf is replaced by an internal node
labeled with the chosen splitting function, which partitions the data
at the node into two new children leaves. Once the tree is built, the
leaves of the tree are labeled by the label of the most common class
that reaches the leaf. Algorithm~\ref{alg:DP-TopDown-DT}, \texttt{DP-TopDown}, is a ``reference implementation" of the
differentially private version of this
algorithm. \texttt{DP-TopDown}, instead of choosing the
best leaf and splitting function, applies the Exponential Mechanism to
noisily select a leaf and splitting function in the built tree so
far. The Exponential Mechanism is applied on the set of all possible
leaves and splitting functions in the current tree; this is
computationally feasible in our Boolean-feature setting. Next we
introduce necessary notation and discuss the privacy guarantee of our
algorithm, and how it is used as a weak learner for our boosting
algorithm.

\paragraph{\textbf{DP TopDown Decision Tree.}} Let $F$ denote a class of
Boolean splitting functions with input domain $\mc{S}$. Each internal
node is labeled by a splitting function $h:\mc{S}\to \{0,1\}$. These
splitting functions route each example $x \in \mc{S}$ to exactly one
leaf of the tree. That is, at each internal node if the splitting
function $h(x)=0$ then $x$ is routed to the left subtree, and $x$ is
routed to the right subree otherwise. Furthermore, let $G$ denote the
\emph{splitting criterion}. $G:[0,1]\to [0,1]$ is a concave function
which is symmetric about $1/2$ and $G(1/2)=1$. Typical examples of
splitting criterion function are Gini and
Entropy. Algorithm~\ref{alg:DP-TopDown-DT} builds decision trees in
which the internal nodes are labeled by functions in $F$, and the
splitting criterion $G$ is used to determine which leaf should be
split next, and which function $h\in F$ should be used for the split.

Let $T$ be a decision tree whose leaves are labeled by $\{0,1\}$ and
$\mu$ be a distribution on $\mc{S}$. The weight of a leaf
$\ell\in leaves(T)$ is defined to be the weighted fraction of data
that reaches $\ell$ i.e., $w(\ell,\mu)=\Pr_{\mu}[x\text{ reaches }
\ell]$. The weighted fraction of data with label $1$ at leaf $\ell$ is
denoted by $q(\ell,\mu)$. Given these we define error of $T$ as follows.
\begin{align*}
    \mathrm{err}(T,\mu) = \sum\limits_{\ell\in leaves(T)} w(\ell,\mu) \min\{q(\ell,\mu), 1-q(\ell,\mu)\}
\end{align*}
Noting that $G(q(\ell,\mu))\geq \min\{q(\ell,\mu), 1-q(\ell,\mu)\}$, we have an upper bound for $\mathrm{err}(T,\mu)$.
 \begin{align*}
    \mathrm{err}(T,\mu) \leq \mc{G}(T,\mu) = \sum\limits_{\ell\in leaves(T)} w(\ell,\mu) G(q(\ell,\mu)).
\end{align*}
For $\ell\in leaves(T)$ and $h\in F$ let $T(\ell,h)$ denote the tree obtained from $T$ by replacing $\ell$ by an internal node that splits subset of data that reaches $\ell$, say $\mc{S}_{\ell}$, into two children leaves $\ell_0$, $\ell_1$. Note that any data $x$ satisfying $h(x) = i$ goes to $\ell_i$. The quality of a pair $(\ell,h)$ is the improvment we achieve by splitting at $\ell$ according to $h$. Formally,
    \[
        \mathrm{im}_{\ell, h, \mu} = \mc{G}(T,\mu)- \mc{G}(T(\ell,h),\mu)
    \]
At each iteration, Algorithm~\ref{alg:DP-TopDown-DT} chooses a pair $(\ell^*,h^*)$ according to the Exponential Mechanism with probability proportional to $\mathrm{im}_{\ell, h,\mu}$. By Theorem~\ref{thm:weighted-EM}, the quality of the chosen pair $(\ell^*,h^*)$ is close to the optimal split with high probability.

\begin{algorithm}[tb]
  \caption{Differentially Private TopDown-DT}\label{alg:DP-TopDown-DT}
  \begin{algorithmic}[1]
    \REQUIRE{Data sample $\mc{S}$, distribution $\hat{\mu}$ over $\mc{S}$, number of internal nodes $t$, and $\eta>0$.}
    \STATE $T\gets$ the single-leaf tree.
    \STATE $\mathcal{C}\gets leaves(T)\times F$ 
    \WHILE{$T$ has fewer than $t$ internal node}
        \STATE $(\ell^*,h^*)\gets$ select a candidate from $\mathcal{C}$ w.p. $\propto\exp(\eta\cdot \mathrm{im}_{\ell,h,\hat{\mu}})$
        \STATE $T\gets T(\ell^*,h^*)$
        \FOR{each new pair $\ell\times h\in leaves(T)\times F$}
                \STATE $\mathrm{im}_{\ell,h,\hat{\mu}}\gets G(T,\hat{\mu})- G(T(\ell,h),\hat{\mu})$
                \STATE Add $\mathrm{im}_{\ell,h,\hat{\mu}}$ to $\mathcal{C}$
        \ENDFOR
    \ENDWHILE 
    \STATE Label leaves by majority label [WRNM with privacy budget $8t\cdot\eta\cdot\zeta$]
    \label{line:leaf-labeling}
    \STATE {\bfseries Output:} $T$
  \end{algorithmic}
\end{algorithm}

\begin{thm}[Privacy guarantee]
    \label{thm:TopDown-privacy}
     \texttt{DP-TopDown}, Algorithm~\ref{alg:DP-TopDown-DT}, is a $(16t\cdot\eta\cdot\zeta, 0, \zeta)$-DP weak learner.
\end{thm}

As before, given a total privacy budget of $\epsilon$, we divide it uniformly
across rounds of boosting. Then, by Theorem \ref{thm:WBregBoost}, we
solve $\epsilon = 16\tau\cdot t\cdot\eta\cdot\zeta$ for $\eta$ to determine how much noise
\texttt{DP-TopDown} must inject at each round.  \texttt{LazyBB} furnishes the promise
that $\zeta \leq 1/\kappa n$. Overall:

\begin{thm}
  \texttt{LazyBB} runs for $\tau$ rounds using \texttt{DP-TopDown} at
  noise rate $\eta = \frac{\epsilon \kappa n}{16\tau t}$ is $\epsilon$-DP.
\end{thm}

\section{Experiments}
\label{sec:experiments}



Here we compare our smooth boosting algorithm (\texttt{LazyBB}) over
both decision trees and 1-Rules to: differentially private logistic
regression using objective perturbation (DP-LR) \cite{DBLP:journals/jmlr/ChaudhuriMS11}, Differentially Private
Bagging (DP-Bag) \cite{DP-Bagging-JordonYS19}, and Privacy-Preserving Gradient
Boosting Decision Trees (DPBoost) \cite{DBLP:conf/aaai/LiWWH20}.
In our implementation we used the IBM differential privacy library (available under MIT licence) \cite{diffprivlib} for standard DP mechanisms and accounting, and scikit-learn (available under BSD licence) for infrastructure \cite{scikit-learn}.
These
experiments show that smooth boosting of \emph{1-Rules} can yield
improved model accuracy and sparsity under identical privacy
constraints. 

We experiment with three freely available real-world
datasets. \textbf{Adult} (Available from UCI Machine Learning
  Repository) has 32,561 training examples, 16,282 test examples, and 162 features after dataset-oblivious one-hot coding --- which incurs
no privacy cost. The task is to predict if someone makes more than 50k
US dollars per year from Census data. Our reported accuracies are holdout tests on the canonical test set associated with Adult. \textbf{Cod-RNA} (available from the LIBSVM website) has 59,535 training
examples and 80 features after dataset-oblivious one-hot coding, and
asks for detection of non-coding RNA. \textbf{Mushroom} (available from the LIBSVM website) has 8124 training examples and 117 features after one-hot coding, which asks to identify poisonous mushrooms. For Mushroom and CodRNA, we report cross-validated estimates of accuracy. All experiments were run on a 3.8 GHz 8-Core Intel Core i7 with 16GB of RAM consumer desktop computer.

\paragraph{\textbf{Parameter Selection Without Assumptions.}}
\begin{table}
\scalebox{0.9}{
\begin{tabular}{ccccc}
    \toprule
    \multirow{2}{*}{WKL}&\multicolumn{3}{c}{\textbf{Parameter}}\\
    \cmidrule{2-4}
    & $\tau$ &  $\lambda$  & $\kappa$\\
    \midrule
     \texttt{OneRule} &  $5, 9, 15, 19, 25, 29, $ 
     & $0.2, 0.25, $ 
     & $0.2, 0.25,$\\ 
     & $39, 49, 65, 75, 99$ & $\cdots, 0.5$ & $ \cdots, 0.5$\\
     \texttt{TopDown} & $5, 9, 15, 19, 25, 29,$ 
     & $0.35, 0.4$ 
     & $0.25, 0.3$ \\ 
     & $35, 39, 45, 51$ & &\\
    \bottomrule
\end{tabular}}
\caption{Parameters grid}
\label{tab:param-grid}
\end{table}

\begin{figure}
    \centering
    \includegraphics[width=0.45\textwidth]{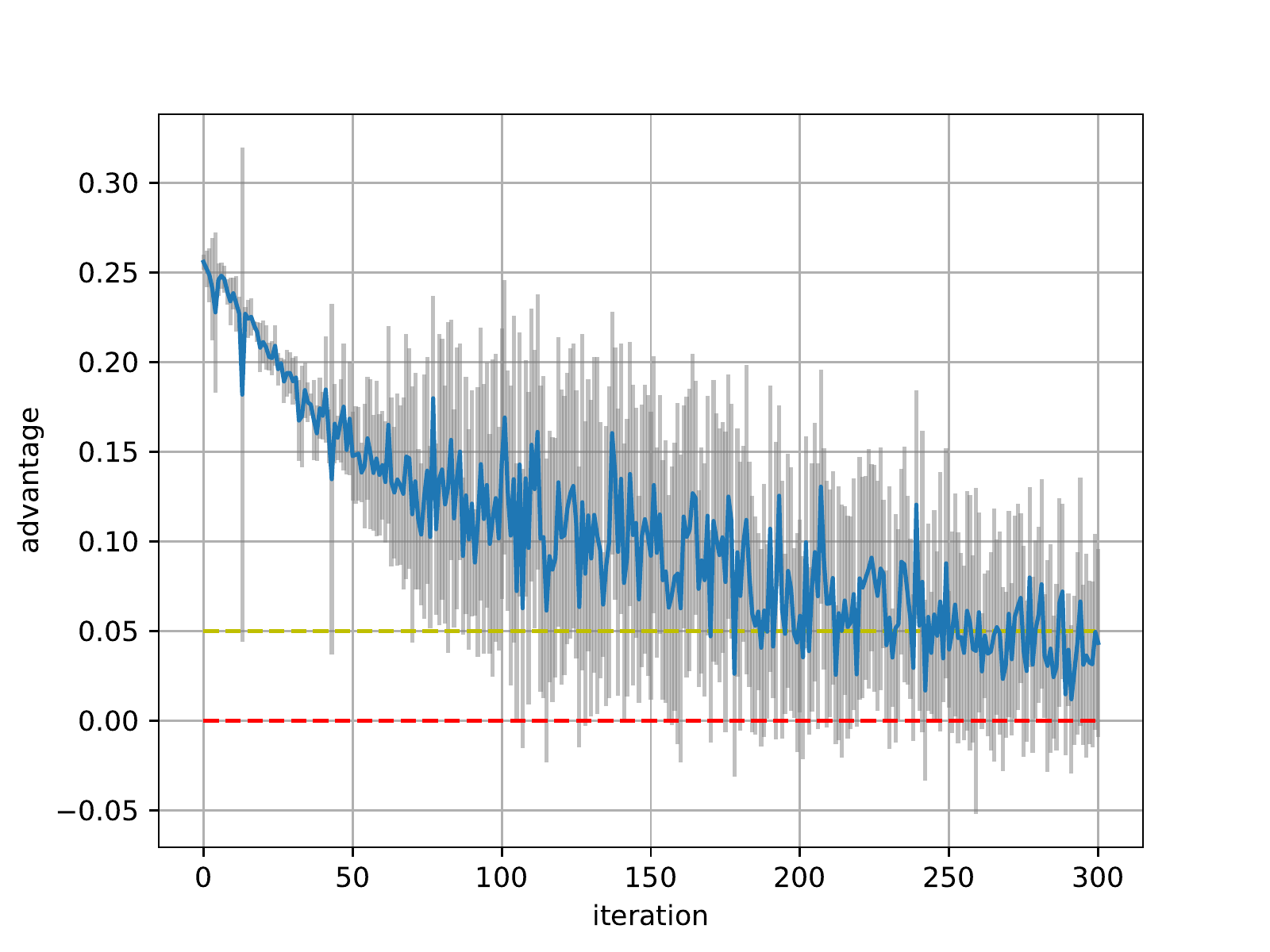}
    \caption{Advantage curve and margin histogram.}
   \label{fig:adv_curve_and_margins}
\end{figure}



We select parameters for \texttt{LazyBB} and \texttt{DP-LR} entirely using grid-search
and cross-validation (Table~\ref{tab:param-grid} for \texttt{LazyBB}) for each value of epsilon plotted i.e. $\epsilon \in (0.05, 0.1,\cdots, 0.5, 1, 3, 5)$. Over the small datasets we use for experiments, the Weak Learner
assumption does not hold for ``long enough'' to realize the training
error guarantee of Theorem~\ref{thm:WBregBoost}. For example, fixing
$\kappa = 1/2$ --- seeking ``good'' margin on only half the training
set --- suppose we have a (1/20)-advantage Weak Learner. That is, at
every round of boosting, each new hypothesis has accuracy at least
55\% over the intermediate distribution. Under these conditions,
Theorem~\ref{thm:WBregBoost} guarantees utility after approximately
4,000 rounds of boosting. Figure~\ref{fig:adv_curve_and_margins} plots
advantage on the Adult dataset at each round of boosting with
$\lambda = \gamma/4$ as required by Theorem~\ref{thm:WBregBoost},
averaged over 10 runs of the boosting decision stumps with total
privacy budget $\epsilon = 1$. The weak learner assumption fails after only  250 rounds of boosting.

And yet, even when run with much
\emph{faster} learning rate $\lambda$, we see good accuracy from
\texttt{LazyBB} --- the assumption holds for \emph{small} $\tau$, ensuring that \texttt{DP-1R} has advantage. So, Theorem~\ref{thm:WBregBoost} is much
more pessimistic than is warranted. This is a know limitation of the
analysis for any \emph{non}-adaptive boosting algorithm
\cite{10.5555/2207821}. In the non-private setting, we set $\lambda$
very slow and boost for ``many'' rounds, until decay in advantage
triggers a stopping criterion. In the private setting (where
non-adaptivity makes differential privacy easier to guarantee) running
for ``many'' rounds is not feasible; noise added for privacy would
saturate the model.  These
experiments motivate further theoretical investigation of boosting
dynamics for non-adaptive algorithms, due to their utility in the
privacy-preserving setting.



\begin{table*}[t]
\centering
\scalebox{0.95}{
\begin{tabular}{c|ccc}

    \rotatebox{90}{~~~{\footnotesize Low privacy regime}} &  \includegraphics[width=0.33\textwidth]{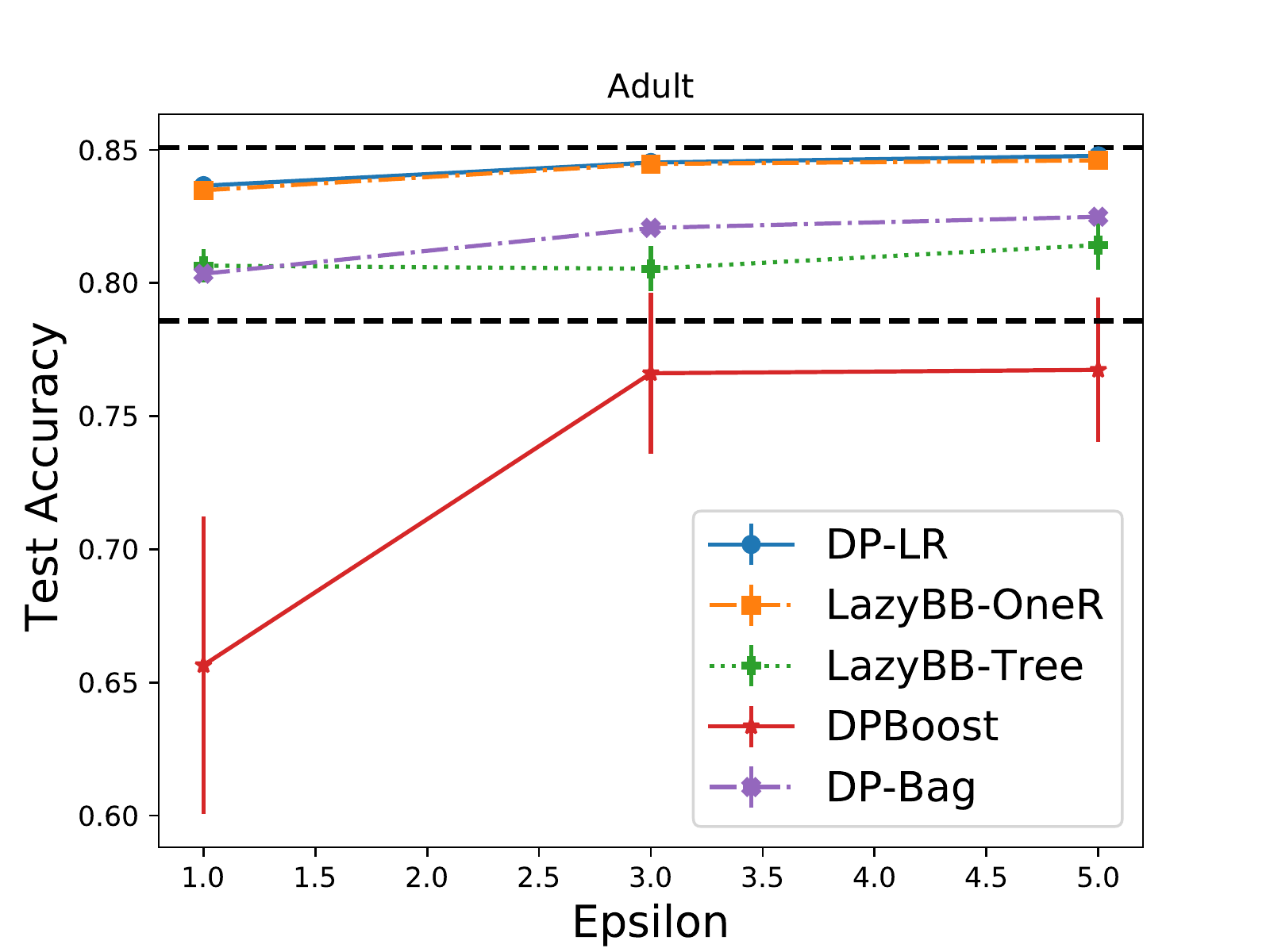} & \hspace{-2 em}
    \includegraphics[width=0.33\textwidth]{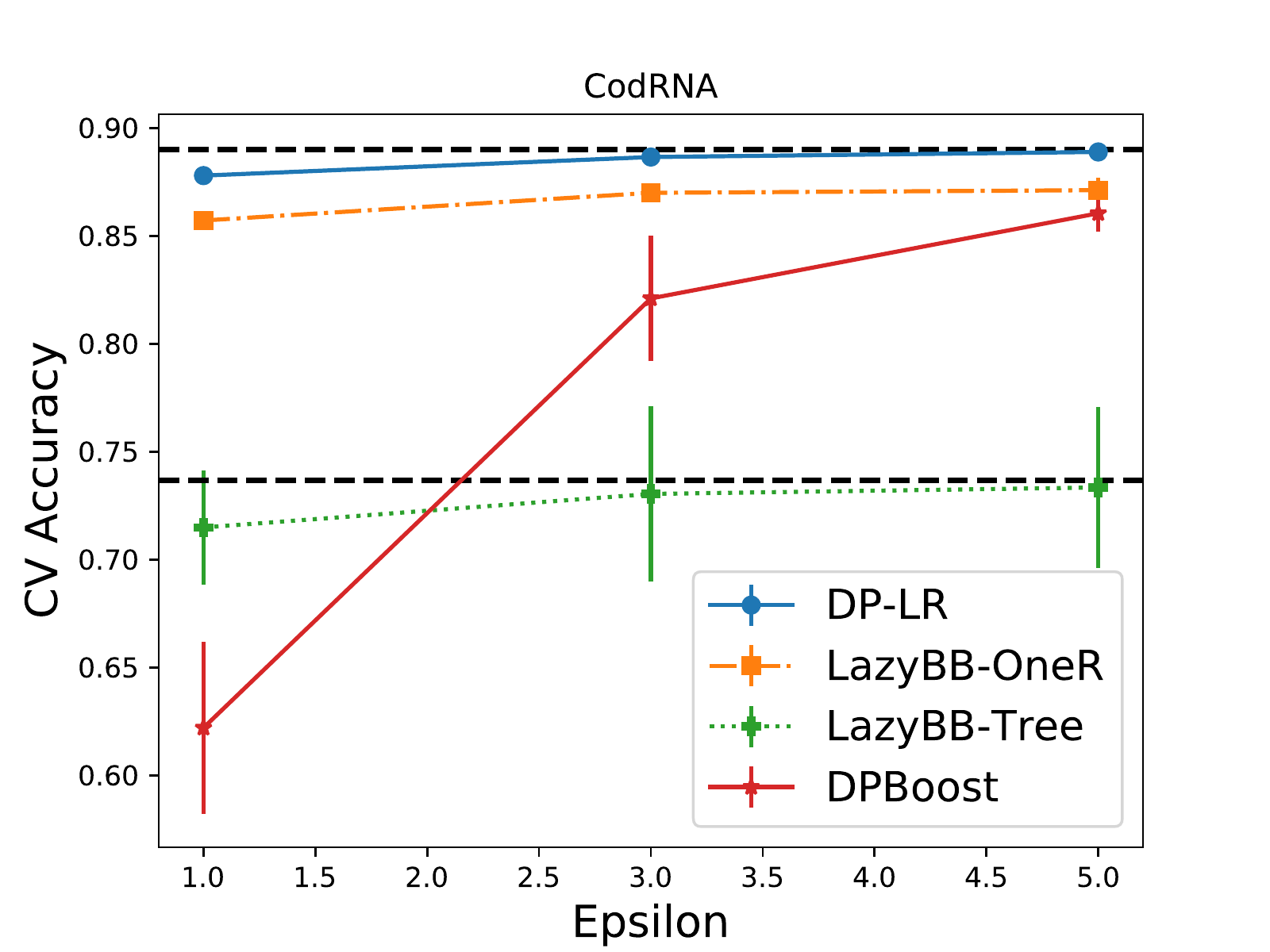} & \hspace{-2 em}
    \includegraphics[width=0.33\textwidth]{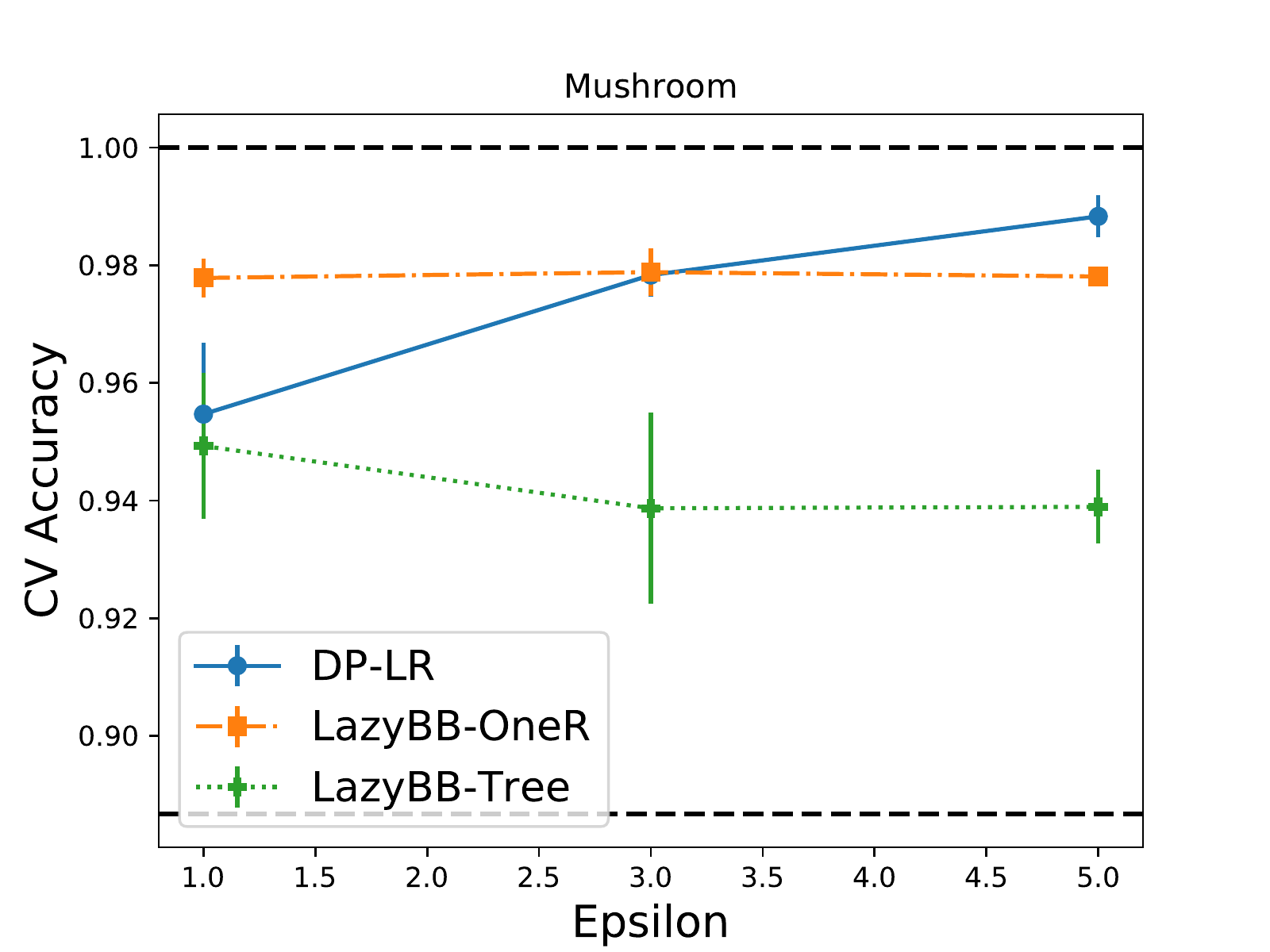} \\
    \midrule 
    \rotatebox{90}{~~~{\footnotesize High privacy regime}} & \includegraphics[width=0.33\textwidth]{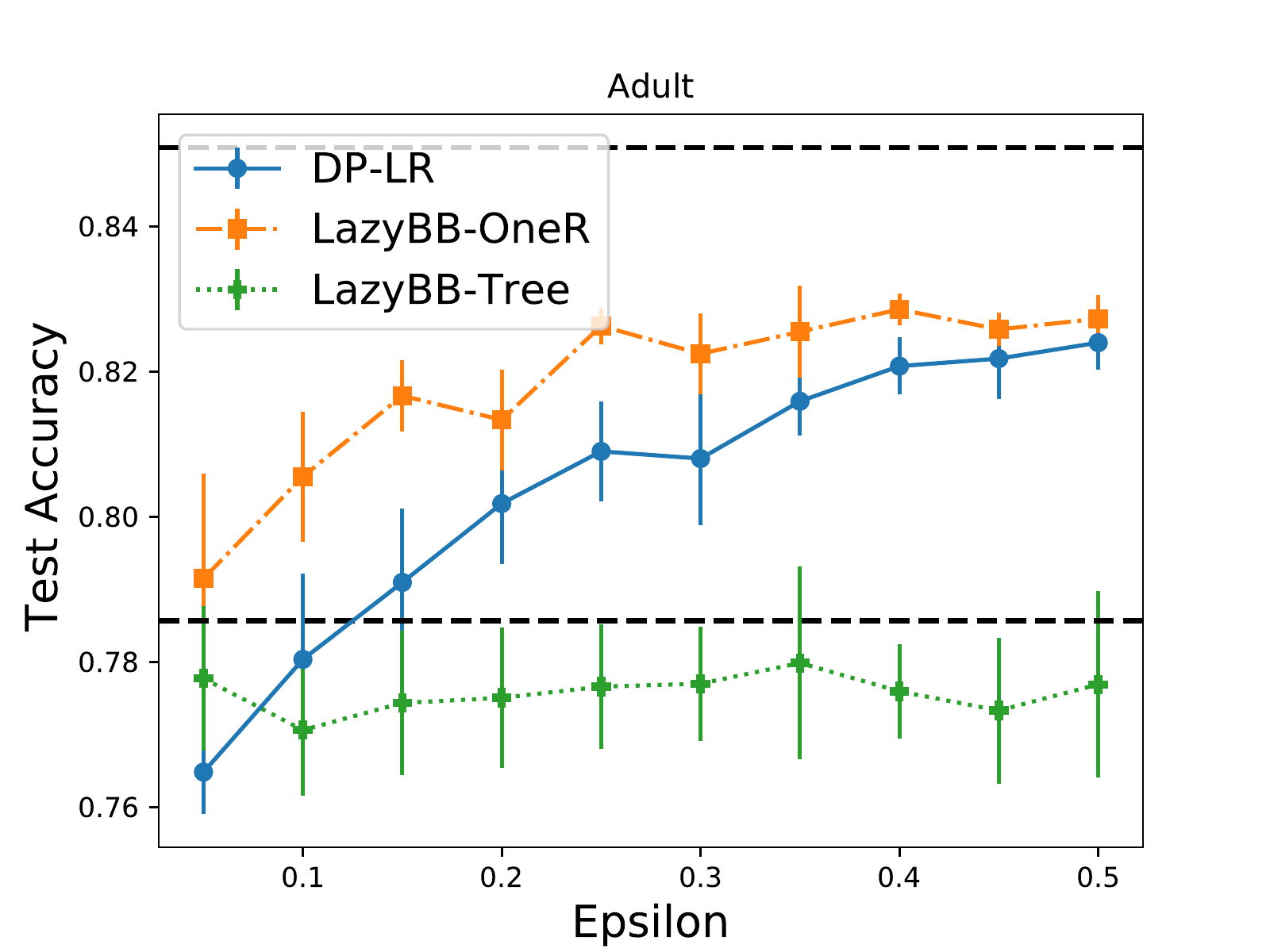} & \hspace{-2 em}
    \includegraphics[width=0.33\textwidth]{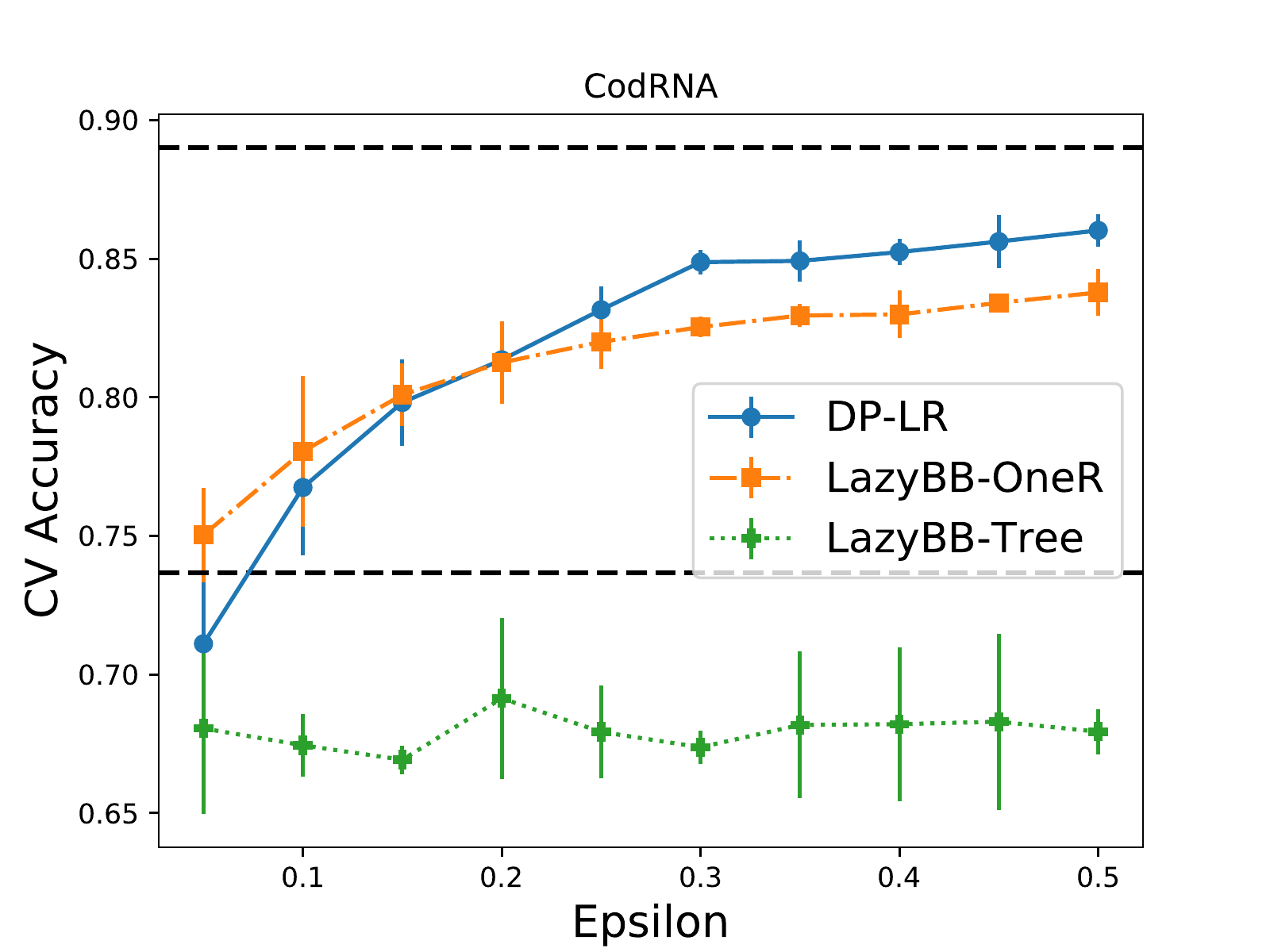} & \hspace{-2 em}
    \includegraphics[width=0.33\textwidth]{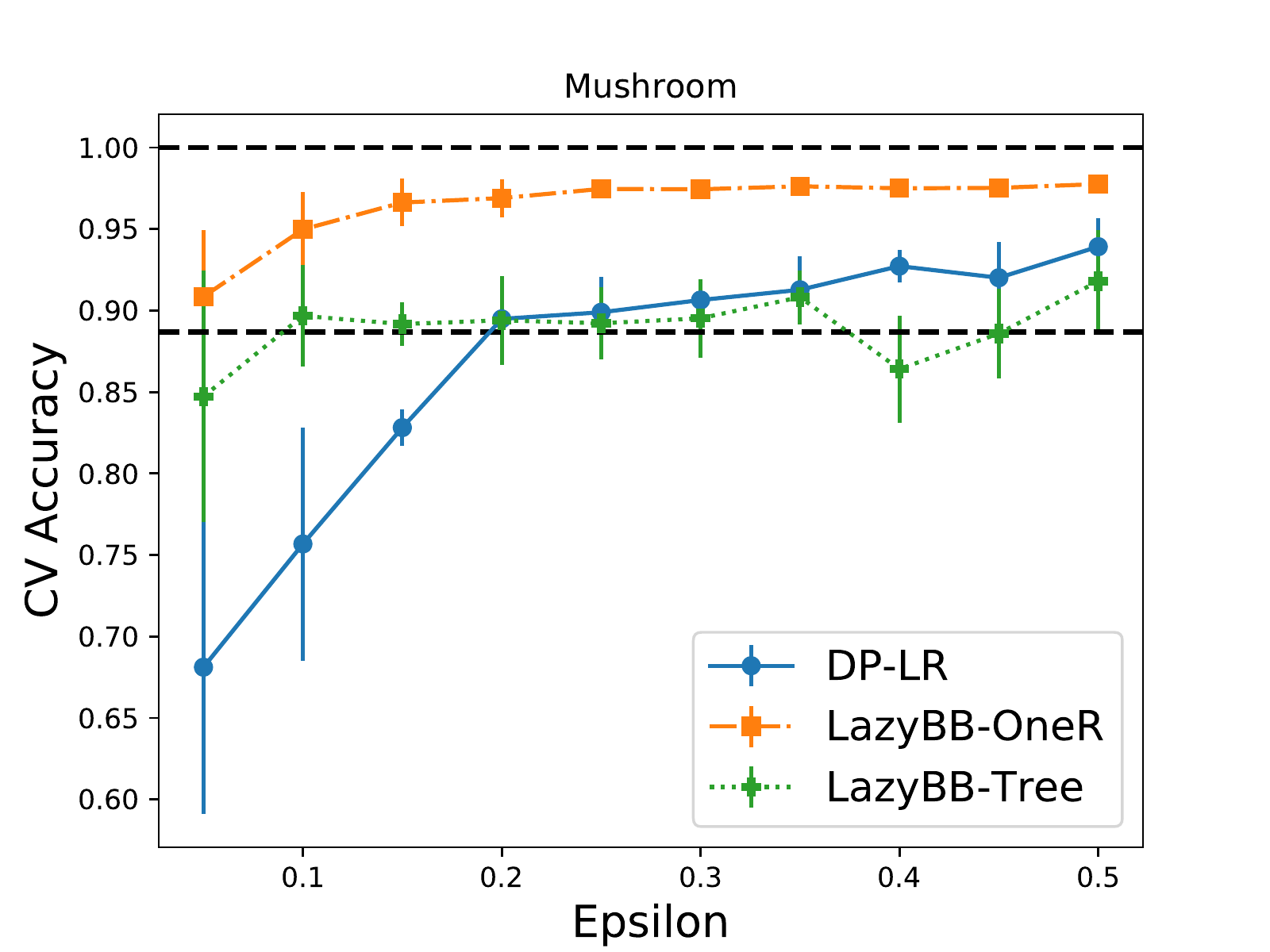}\\
\end{tabular}}
\caption{Learning Curves --- Privacy vs. accuracy.}
\label{table-test-1}
\end{table*}


\paragraph{\textbf{Results.}}
\label{sec:results}
In Table \ref{table-test-1} we plot the accuracy of each of the 5 methods
above against privacy constraint $\epsilon$, along with two
non-private baselines to both quantify the ``cost of privacy'' and
ensure that the private learners are non-trivial. The strong
non-private baseline is the implementation of Gradient Boosted Trees
in sklearn, the weak non-private baseline is a single 1-Rule. It is important to note that for DP-Bag and DPBoost, we only compare our results for datasets and regimes that the corresponding hyperparameters are reported in the related works. Surprisingly, we found that \texttt{LazyBB} over 1-Rules and
differentially private logistic regression were the best performing
models --- despite being the \emph{simplest} algorithms to state,
reason about, and run.

\paragraph{\textbf{Sparsity, regularization, and interpretability.}}
Algorithms used for high-stakes decisions should be both well-audited and privacy-preserving. However, often there is a trade-off between privacy and interpretability \cite{harder2020private_interpretability}. Generally, noise injected to protect privacy harms interpretability. But our algorithms maintain accuracy under strong privacy constrains while admitting a high level of sparsity --- which facilitates interpretability. Table~\ref{tab:sparsity_measures} lists measurements across different levels of privacy. For an example of boosted one-rules at $\epsilon=0.4$ DP, see Table~\ref{tab:example_boosted_onerule}.


\texttt{DP-LR} --- another simple algorithm with excellent performance --- uses $L_2$ regularization to improve generalization. While $L_2$ regularization keeps total mass of weights relatively small, it generally assigns non-negligible weight to \emph{every} feature. Hence, the resulting model becomes less interpretable as the dimension of data grows. On other hand, \texttt{LazyBB} with 1-Rules controls sparsity by the number of rounds of boosting. Just as with non-private non-adaptive boosting algorithms, we can see this as a greedy approximation to $L_1$ regularization of a linear model \cite{DBLP:journals/jmlr/RossetZH04}. Moreover, the final model can be interpreted as a simple integral weighted voting of features. 

\begin{table}
    \centering
    \scalebox{0.9}{
    \begin{tabular}{ cccc } 
    \toprule
    $\epsilon$ & features count mean & features count std &  \% features\\  \midrule
    0.40 &  6.4 & 0.800 &  3.95\% \\ 
    0.50 & 12.8 & 0.400 &  7.90\% \\ 
    1.00 & 30.6 & 1.200 &  18.88\% \\ 
    3.00 & 72.8 & 2.481 &  44.93\% \\ 
    5.00 & 49.8 & 2.785 &  30.74\% \\ \bottomrule
\end{tabular}}
    \caption{Statistics of number of features used by \texttt{LazyBB} with \texttt{DP-1R} across different levels of privacy on adult dataset. See the Appendix for the complete table.}
\label{tab:sparsity_measures}
\end{table}

\begin{table}
    \centering
    \begin{tabular}{ cc }
 \toprule
 votes & (feature, value) \\  \midrule
 ~3 & marital-status : Married-civ-spouse  \\  
 -2 & capital-gain = 0 \\  
 ~1 & occupation : Exec-managerial\\  
 ~1 & occupation : Prof-specialty \\  
 ~1 & 13 $\leq$ education-num $\leq$ 14.5\\  
 -1 & age $\leq$ 17\\
 \bottomrule
\end{tabular}
    \caption{A $0.4$-DP model obtained by training \texttt{LazyBB} with \texttt{DP-1R} on adult dataset with $0.82$ accuracy.}
\label{tab:example_boosted_onerule}
\end{table}

\paragraph{\textbf{Pessimistic Generalization Theory.}} Empirically,
\texttt{LazyBB} generalizes well. As with AdaBoost, we could try to explain this with large margins and Rademacher complexity, which applies to any voting classifier. So, we estimated the Rademacher complexity of 1-Rules over each dataset to predict test error. The bounds are far more pessimistic than the experiments; please see the Appendix for comparison tables. Intuitively, if \texttt{LazyBB} showed larger margins on the training data than on unseen data, this would constitute a \emph{membership inference attack} --- which is ruled out by differential privacy. This motivates theoretical investigation of new techniques to guarantee generalization of differentially private models trained on small samples.

\section{Weighted Exponential Mechanism: Proof of Theorem~\ref{thm:weighted-EM}}

In this section we discuss the privacy guarantee of the Weighted
Exponential Mechanism defined in Definition~\ref{def:WEM}. Our proof
follows the same steps as the standard Exponential Mechanism
\cite{mcsherry2007mechanism}. Our goal is to prove, given two
neighboring datasets and two similar distributions on them, the
Weighted Exponential Mechanism outputs the same hypothesis with high
probability.

In what follows let $M$ denote the Weighted Exponential Mechanism, and let $\mathcal{H}=\mr{Range}(M)$. Suppose $\mc{S},\mc{S}'$ are two neighboring datasets of size $n$ and $\mu,\mu'$ are distributions over $[n]$ such that $\mathtt{d}(\mu,\mu')<\zeta$. Furthermore, let $q_{D,\mu} \colon \mathcal{H}\to \R$ be a quality score that has robust sensitivity $\Delta_\zeta$. That is, for every hypothesis $h\in \mathcal{H}$, we have
\begin{align}
\label{eq:q-delta}
    \max\limits_{\substack{D\sim D' \\ \mu,\mu': \mathtt{d}(\mu,\mu') < \zeta}}  |q_{D,\mu}(h)-q_{D',\mu'}(h)|\leq \Delta_\zeta.
\end{align}

We proceed to prove that for any $h\in \mathcal{H}$ the following holds
    \[
        \Pr[M(S,\mu) = h ] \le
        \exp(2\eta\Delta_\zeta) \Pr[M(S',\mu') = h].
    \]
Recall that $M$ outputs a hypothesis $h$ with probability proportional to $\exp(\eta\cdot q_{D,\mu})$ with $\eta = \frac{\epsprivacy}{2\Delta_\zeta}$. Let us expand the probabilities above,
    \begin{align*}
        \frac{\Pr[M(\mc{S},\mu) = h ]}{\Pr[M(\mc{S}',\mu') = h]} &= \frac{\exp{(\eta\cdot q_{\mc{S},\mu}(h))}}{\exp{(\eta\cdot q_{\mc{S}',\mu'}(h))}}\\
        &\times \frac{\sum\limits_{h\in\mc{H}}\exp{(\eta\cdot q_{\mc{S}',\mu'}(h))}}{\sum\limits_{h\in\mc{H}}\exp{(\eta\cdot q_{\mc{S},\mu}(h))}}.
   \end{align*}
    Consider the first term, then    
      \begin{align*}
        &\frac{\exp{(\eta\cdot q_{\mc{S},\mu}(h))}}{\exp{(\eta\cdot q_{\mc{S}',\mu'}(h))}}
        =\exp{(\eta[q_{\mc{S},\mu}(h)-q_{\mc{S}',\mu'}(h)])}\\
        & \leq \exp{(\eta\cdot \Delta_\zeta)}
          \tag{By~\eqref{eq:q-delta}}
    \end{align*}
    Now consider the second term, then

    \begin{align*}
        &\frac{\sum\limits_{h\in\mc{H}}\exp{(\eta\cdot q_{\mc{S}',\mu'}(h))}}{\sum\limits_{h\in\mc{H}}\exp{(\eta\cdot q_{\mc{S},\mu}(h))}} \leq \frac{\sum\limits_{h\in\mc{H}}\exp{(\eta\cdot[ q_{\mc{S},\mu}(h)+\Delta_\zeta])}}{\sum\limits_{h\in\mc{H}}\exp{(\eta\cdot q_{\mc{S},\mu}(h))}}\\
        & = 
        \frac{\exp{(\eta \Delta_\zeta)}\sum\limits_{h\in\mc{H}}\exp{(\eta\cdot q_{\mc{S},\mu}(h))}}{\sum\limits_{h\in\mc{H}}\exp{(\eta\cdot q_{\mc{S},\mu}(h))}}\\
        &=\exp{(\eta \Delta_\zeta)}
    \end{align*}
    Hence, it follows that 
    \begin{align*}
        &\frac{\Pr[M(\mc{S},\mu) = h ]}{\Pr[M(\mc{S}',\mu') = h]}\leq \exp{(\eta \Delta_\zeta)}\cdot \exp{(\eta \Delta_\zeta)}\\
        &= \exp{(2\eta \Delta_\zeta)}
    \end{align*}
    This implies that, for $\eta>0$, WEM is a $(2\eta\Delta_\zeta,0,\zeta)$-differentially private weak learner. (Note that setting $\eta=\frac{2\epsprivacy}{2\Delta_\zeta}$ yields a $(\epsprivacy,0,\zeta)$-differentially private weak learner.)

We point out that the proof for the utility guarantee of Theorem~\ref{thm:weighted-EM} is identical to the proof of the utility guarantee in standard Exponential Mechanism \cite{mcsherry2007mechanism}.

\section{Weighted Return Noisy Max: Proof of Theorem~\ref{thm:weighted-RNM}}

In this section we discuss the privacy guarantee of the Weighted Return Noisy Max defined in Section~\ref{sec:DP-Learning}. Our proof follows the same steps as the standard Return Noisy Max explained in \cite{dwork2014algorithmic} with slight modification. Our goal is to prove, given two neighboring datasets and two similar distributions on them, the WRNM outputs the same hypothesis index.

Let $f_1,\dots,f_k$ be $k$ quality functions where each $f_i:\mc{S}\times \cD(S)\to \R$ maps datasets and distributions over them to real numbers. For a dataset $S$ and distribution $\mu$ over $S$, WRNM adds independently generated Laplace noise $Lap(1/\eta)$ to each $f_i$ and returns the index of the largest noisy function i.e. $i^*=\argmax\limits_{i} (f_i+Z_i)$ where each $Z_i$ denotes a random variable drawn independently from the Laplace distribution with scale parameter $1/\eta$. In what follows let $M$ denote the WRNM.

Suppose $\mc{S},\mc{S}'$ are two neighboring datasets of size $n$ and $\mu,\mu'$ are distributions over $[n]$ such that $\mathtt{d}(\mu,\mu')<\zeta$. Furthermore, suppose each $f_i$ has robust sensitivity at most $\Delta_\zeta$. That is, for every index $i\in \{1,\dots,k\}$, we have
\begin{align}
\label{eq:f-delta}
    \max\limits_{\substack{D\sim D' \\ \mu,\mu': \mathtt{d}(\mu,\mu') < \zeta}}  |f_{i}(\mc{S},\mu)-f_{i}(\mc{S}',\mu')|\leq \Delta_\zeta.
\end{align}

Fix any $i\in\{1,\dots,k\}$. We will bound the ratio of the probabilities that $i$ is selected by $M$ with inputs $\mc{S},\mc{S}'$ and distributions $\mu,\mu'$.

Fix $Z_{-i}=(Z_1,\dots,Z_{i-1},Z_{i+1},\dots,Z_k)$, where each $Z_j\in Z_{-i}$ is drawn from $Lap(1/\eta)$. We first argue that 
\[
    \frac{\Pr[M(\mc{S},\mu)=i \mid Z_{-i}]}{\Pr[M(\mc{S}',\mu')=i \mid Z_{-i}]}\leq \mr{e}^{2\eta\cdot \Delta_\zeta} .
\]

Define $Z^*$ to be the minimum $Z_i$ such that 
\[
    f_i(\mc{S},\mu) + Z^* > f_j(\mc{S},\mu) + Z_j ~~~\forall j\neq i
\]
Note that, having fixed $Z_{-i}$, $M$ will output $i$ only if $Z_i \geq Z^*$. Recalling \eqref{eq:f-delta}, for all $j\neq i$, we have the following,
\begin{align*}
    f_i(\mc{S}',\mu')+Z^*+\Delta_\zeta \geq f_i(\mc{S},\mu)+Z^* > f_j(\mc{S},\mu)+Z_j \\\geq f_j(\mc{S}',\mu')+Z_j-\Delta_\zeta
\end{align*}
This implies that 
\begin{align*}
    f_i(\mc{S}',\mu')+Z^*+2\Delta_\zeta \geq f_j(\mc{S}',\mu')+Z_j
\end{align*}

Now, for dataset $\mc{S}'$, distribution $\mu'$, and $Z_{-i}$, mechanism $M$ selects the $i$-th index if $Z_i$, drawn from $Lap(1/\eta)$, satisfies $Z_i \geq Z^*+2\Delta_\zeta$.
\begin{align*}
    &\Pr_{Z_i\sim Lap(1/\eta)}[M(\mc{S}',\mu') = i \mid Z_{-i}]\\ 
    &\geq \Pr_{Z_i\sim Lap(1/\eta)}[Z_i \geq Z^*+2\Delta_\zeta] \\
    &\geq \mr{e}^{-(2\eta\Delta_\zeta)}\Pr_{Z_i\sim Lap(1/\eta)}[Z_i \geq Z^*] \\
    &= \mr{e}^{-(2\eta\Delta_\zeta)}\Pr_{Z_i\sim Lap(1/\eta)}[M(\mc{S},\mu) = i \mid Z_{-i}]
\end{align*}
Multiplying both sides by $\mr{e}^{(2\eta\Delta_\zeta)}$ yields the desire bound.

\begin{align*}
    \frac{\Pr_{Z_i\sim Lap(1/\eta)}[M(\mc{S},\mu) = i \mid Z_{-i}]} {\Pr_{Z_i\sim Lap(1/\eta)}[M(\mc{S}',\mu') = i \mid Z_{-i}]} \leq \mr{e}^{(2\eta\Delta_\zeta)}
\end{align*}
This implies that, for $\eta>0$, WRNM is a $(2\eta\Delta_\zeta,0,\zeta)$-differentially private weak learner. (Note that setting $\eta=\frac{2\epsprivacy}{2\Delta_\zeta}$ yields a $(\epsprivacy,0,\zeta)$-differentially private weak learner.)
\section{Weak Learner: DP 1-Rules}
\label{sec:dp-one-rules}

Throughout this section let $\mathcal{S}=\{(\mb{x}_i,y_i)\}^n$ where
$\mb{x}_i=(x_{i1},\dots , x_{ir})$ denote a dataset, and let $\mu$ be
a distribution over $[n]$. We will brute-force ``1-Rules,'' also known
as Decision Stumps \cite{Decision-Stump-L92,Decision-Stump-Holte93}. Here, these simply evaluate a single Boolean
literal such as $\neg x _{17}$ --- an input variable that may or may
not be negated. We also admit the constants \textsf{True} and
\textsf{False} as literals.

A brutally simple but surprisingly effective weak learner returns the
literal with optimal weighted agreement to the labels. For any 1-Rule
$h$ define $\mr{err}(\mc{S}, \mu, h)$ to be:

\[
  \mr{err}(\mc{S}, \mu, h) = \sum\limits_{(\mb{x}_i,y_i)\in \mc{S}}
  \mu(i)\chi\{ h(\mb{x}_i) \neq y\}.
\]

For learning 1-Rules under DP constraints, the natural approach is to
use the Exponential Mechanism to noisily select the best possible
literal. There is a small type error: the standard Exponential
Mechanism does not consider utility functions with an auxiliary
weighting $\mu$. But for weak learning we only demand privacy (close
output distributions) when \emph{both} the dataset and measures are
``close.'' When both promises hold and $\mu$ is fixed, the Exponential
Mechanism is indeed a differentially private 1-Rule learner. We show
this formally below.

\begin{obs}
\label{err-sensitivity}
  Let $\mc{S}\sim \mc{S}'$ be any two neighboring datasets and set $I=\mc{S}\cap \mc{S}'$. Then, for any two distributions $\mu,\mu'$ over $[n]$, we have
  \begin{align*}
    &\left|\mr{err}(\mc{S}, \mu, T) - \mr{err}(\mc{S}', \mu', T)\right| \\
    &=
    \begin{multlined}[t]
     \Bigg|\sum\limits_{(\mb{x}_i,y_i)\in \mc{S}} \mu(i)\chi\{ T(\mb{x}_i) \neq y\} \\
     -\sum\limits_{(\mb{x}_i,y_i)\in \mc{S}'} \mu'(i)\chi\{ T(\mb{x}_i) \neq y\}\Bigg|
    \end{multlined}
    \\ 
    & =
    \begin{multlined}[t]
     \Bigg|\sum\limits_{(\mb{x}_i,y_i)\in \mc{S}\cap\mc{S}'} [\mu(i)-\mu'(i)]\chi\{ T(\mb{x}_i) \neq y\}  \\  +\sum\limits_{(\mb{x}_i,y_i)\in \mc{S}\triangle\mc{S}'} [\mu(i)-\mu'(i)]\chi\{ T(\mb{x}_i) \neq y\}\Bigg|
    \end{multlined}\\
    & \le \sum\limits_{(\mb{x}_i,y_i)\in \mc{S}\cap\mc{S}'} |\mu(i)-\mu'(i)| +\sum\limits_{(\mb{x}_i,y_i)\in \mc{S}\triangle\mc{S}'} |\mu(i)-\mu'(i)| \\
    &= \sum\limits_{i=1}^{n} |\mu(i)-\mu'(i)|=2\mathtt{d}(\mu,\mu'). 
  \end{align*}
\end{obs}

\begin{algorithm}[t]
  \caption{Differentially Private 1-Rule
    Induction($\mc{S},\mu,\eta$)}
  \label{alg:DPMDT}
  \begin{algorithmic}[1]
    \REQUIRE{Dataset $\mc{S}$, distribution $\mu$ over $[1,\dots,|\mc{S}|]$, and $\eta > 0$.}
    \STATE Let $\mathcal{H}$ be the set of all literals over $\mc{S}$
    plus the constants \textsf{True} and \textsf{False}
    \FOR{$h \in \mathcal{H}$}
    \STATE $q_{\mc{S},\mu}(h) \gets -\mr{err}(\mc{S},\mu,h)$.
    \ENDFOR
    \STATE $h_{out}\gets$ select a hypothesis $h \in \mathcal{H}$ with probability proportional to $\exp{(\eta \cdot q_{\mc{S},\mu}(h))}$
    \STATE \textbf{return} $h_{out}$
  \end{algorithmic}
\end{algorithm}

\begin{thm}\label{thm:privacy}
  Algorithm~\ref{alg:DPMDT} is a $(4 \eta \zeta,0,\zeta)$-differentially private weak learner.
\end{thm}
\begin{proof}
    Suppose $\mc{S},\mc{S}'$ are two neighboring datasets of size $n$ and $\mu,\mu'$ are distributions over $[n]$ such that $\mathtt{d}(\mu,\mu')<\zeta$. Observation~\ref{err-sensitivity} tells us that the quality score $q_{\mc{S},\mu}(h)=\mr{err}(\mc{S},\mu,h)$ has robust sensitivity $2\zeta$. Hence, by Theorem~\ref{thm:weighted-EM}, we have that Algorithm~\ref{alg:DPMDT} is a $(4 \eta \zeta,0,\zeta)$-differentially private weak learner.
\end{proof}

\begin{thm}\label{thm:optimality}
   Let $h_{opt}$ denote the optimal hypothesis in $\mc{H}$. Then Algorithm~\ref{alg:DPMDT}, with probability at least $1-\beta$, returns $h_{out}\in\mc{H}$ such that
  \[
    \mr{err}(h_{out})  \leq \mr{err}(h_{opt}) +  \frac{1}{\eta}\ln{\frac{|\mc{H}|}{\beta}}.
  \]
\end{thm}

\begin{proof}
    By Theorem~\ref{thm:weighted-EM}, with probability at least $1-\beta$, we have
    \begin{align}
        q_{\mc{S},\mu}(h_{out}) \geq \max\limits_{h\in \mc{H}}q_{\mc{S},\mu}(h) - \frac{1}{\eta}\ln{\frac{|\mc{H}|}{\beta}}\label{utility-1}
     \end{align}
    Note that $q_{\mc{S},\mu}(h)=-\mr{err}(\mc{S},\mu,h)$ for all $h\in \mc{H}$ and $\max\limits_{h\in \mc{H}}q_{\mc{S},\mu}(h)=-\mr{err}(h_{opt})$. This gives us 
    \begin{align*}
    &-\mr{err}(h_{out})  \geq -\mr{err}(h_{opt}) -  \frac{1}{\eta}\ln{\frac{|H|}{\beta}}\implies\\
        & \mr{err}(h_{out})  \leq \mr{err}(h_{opt}) +  \frac{1}{\eta}\ln{\frac{|\mc{H}|}{\beta}}.
    \end{align*}
\qedhere
\end{proof}
As we already discussed, in order to  construct PAC learners by boosting weak learners we need weak learners that only beat random guessing on any distribution
over the training set. Here, we wish to use Algorithm~\ref{alg:DPMDT} as a weak learner. That is, we show that Algorithm~\ref{alg:DPMDT} (with high probability) is better than random guessing. In what follows we have Theorem~\ref{thm:1R-advantage} and its proof.
\begin{thm}
    Under a weak learner assumption with advantage  $\gamma$, Algorithm~\ref{alg:DPMDT}, with probability at least $1-\beta$, is a weak learner with advantage at least $\gamma -\frac{1}{\eta}\ln{\frac{|\mc{H}|}{\beta}}$. That is, for any distribution $\mu$ over $\{1,\dots,|\mc{S}|\}$, we have  
    \[
        \sum\limits_{(\mb{x}_i,y_i)\in \mc{S}} \mu(i) \chi\{ h_{out}(\mb{x}_i) \neq y\}  \leq 1/2 -\left(\gamma - \frac{1}{\eta}\ln{\frac{|\mc{H}|}{\beta}}\right).
    \]
\end{thm}
\begin{proof}
     By Theorem~\ref{thm:optimality}, Algorithm~\ref{alg:DPMDT} with probability at least $1-\beta$ outputs a hypothesis $h_{out}$ such that 
    \[
        \mr{err}(h_{out})  \leq \mr{err}(h_{opt}) +  \frac{1}{\eta}\ln{\frac{|\mc{H}|}{\beta}}.
    \]
    Under a \emph{weak learner assumption}, we assume that an optimal hypothesis $h_{opt}$ is at least as good as random guessing. That is $\mr{err}(h_{opt})< 1/2-\gamma$. This yields the desired result.
\end{proof}

 \section{Proof of Theorem~\ref{thm:TopDown-privacy}}
 Here we consider splitting criterion to be the Gini criterion $G(q)=4q(1-q)$. Note that this function is symmetric about $1/2$ and $G(1/2)=1$. Throughout this section,  $\mc{S}\sim\mc{S}'$ are two neighboring datasets of size $n$ and $\mu,\mu'$ are distributions over $[n]$ such that $\mathtt{d}(\mu,\mu')<\zeta$. Observe that for a decision tree $T$ we have  $|w(\ell,\mu)-w(\ell,\mu')|\leq \zeta$ and $|q(\ell,\mu)-q(\ell,\mu')|\leq \zeta$. Before proceeding to  provide an upper bound on the sensitivity of $\mr{im}_{\ell,h,\mu}$, we prove some useful lemmas.
 \begin{lem}
 \label{Lem:sens-single-node}
     The following holds.
    \begin{multlined}[t]
        4\Big|w(\ell,\mu)q(\ell,\mu)(1-q(\ell,\mu)) \\- w(\ell,\mu')q(\ell,\mu)(1-q(\ell,\mu'))\Big| \leq \frac{5}{4}\zeta
    \end{multlined}
 \end{lem}
 \begin{proof}
    As the Gini criterion $G(q)=4q(1-q)$ is symmetric about $1/2$, without loss of generality, we assume $q(\ell)\leq 1/2$. Furthermore, suppose $w(\ell,\mu)q(\ell,\mu)(1-q(\ell,\mu))$ is greater than $w(\ell,\mu')q(\ell,\mu')(1-q(\ell,\mu'))$. The arguments for the other cases are analogous. 
    \begin{align*}
        & w(\ell,\mu)q(\ell,\mu)(1-q(\ell,\mu)) - w(\ell,\mu')q(\ell,\mu')(1-q(\ell,\mu'))\\
        &\leq 
        w(\ell,\mu)q(\ell,\mu)(1-q(\ell,\mu)) \\
        &\quad - w(\ell,\mu')(q(\ell,\mu)-\zeta)(1-q(\ell,\mu)+\zeta)
        \\
        & = 
        w(\ell,\mu)q(\ell,\mu)(1-q(\ell,\mu)) \\
        &\quad - w(\ell,\mu')q(\ell,\mu)(1-q(\ell,\mu)+\zeta)\\
        &\quad + w(\ell,\mu')\zeta(1-q(\ell,\mu)+\zeta)
        \\
        & \leq  w(\ell,\mu)q(\ell,\mu)(1-q(\ell,\mu)) \\
        &\quad - w(\ell,\mu')q(\ell,\mu)(1-q(\ell,\mu))+ \zeta
        \\
        &\leq |w(\ell,\mu)-w(\ell,\mu')|q(\ell,\mu)(1-q(\ell,\mu)) + \zeta\leq \frac{5}{4}\zeta
    \end{align*}
 \end{proof}
 \begin{lem}
     \label{lem:im-sensitivity}
      For a decision tree $T$ and $(\ell,h)\in leaves(T)\times F$ we have
     \[
        \Big| \mr{im}_{\ell,h,\mu}(\mc{S})-\mr{im}_{\ell,h,\mu'}(\mc{S}') \Big| \leq 4\zeta.
     \]
 \end{lem}

 \begin{proof}
     For dataset $\mc{S}$ let $\mc{G}(T)=\sum\limits_{\ell\in leaves(T)}w(\ell)G(q(\ell))$. Recall the definition of $\mr{im_{\ell,h,\mu}}$,
     \begin{align*}
         \mr{im}_{\ell,h,\mu}(\mc{S}) 
         &= \mc{G}(T,\mu)-\mc{G}(T(\ell,h),\mu)\\
         & = w(\ell,\mu) G(q(\ell,\mu))-w(\ell_0,\mu)G(q(\ell_0,\mu))\\
         &-w(\ell_1,\mu)G(q(\ell_1,\mu))
     \end{align*}
     Similarly, for dataset $\mc{S}'$ let $\mc{G}(T,\mu')=\sum\limits_{\ell\in leaves(T)}w(\ell,\mu')G(q(\ell,\mu'))$. Then we have
     \begin{align*}
         & \mr{im}_{\ell,h,\mu'}(\mc{S}') = \mc{G}(T,\mu')-\mc{G}(T(\ell,h),\mu')\\
         & = w(\ell,\mu') G(q(\ell,\mu'))-w(\ell_0,\mu')G(q(\ell_0,\mu'))\\
         &\quad -w(\ell_1,\mu')G(q(\ell_1,\mu'))
     \end{align*}
     Having these we can rewrite $\Big| \mr{im}_{\ell,h,\mu}(\mc{S})-\mr{im}_{\ell,h,\mu'}(\mc{S}') \Big|$ as follows,
     \begin{align*}
         & \Big| \mr{im}_{\ell,h,\mu}(\mc{S})-\mr{im}_{\ell,h,\mu'}(\mc{S}') \Big|\\
         &= \Big|\mc{G}(T,\mu)-\mc{G}(T(\ell,h),\mu)-\mc{G}(T,\mu')+\mc{G}(T(\ell,h),\mu')\Big|\\
         & =
         \Bigg| w(\ell,\mu) G(q(\ell,\mu))-w(\ell_0,\mu)G(q(\ell_0,\mu)) \\
         & \quad -w(\ell_1,\mu)G(q(\ell_1,\mu)) -
         w(\ell,\mu')G(q(\ell,\mu'))\\
         &\quad +w(\ell_0,\mu')G(q(\ell_0,\mu'))+w(\ell_1,\mu')G(q(\ell_1,\mu'))
         \Bigg|
         \\
         &\leq 
         \Big| w(\ell,\mu) G(q(\ell,\mu))-w(\ell,\mu')G(q(\ell,\mu'))\Big|\\ &\quad+\Big|w(\ell_0,\mu')G(q(\ell_0,\mu'))-w(\ell_0,\mu)G(q(\ell_0,\mu))\Big|\\
         &\quad+\Big|w(\ell_1,\mu')G(q(\ell_1,\mu'))-w(\ell_1,\mu)G(q(\ell_1,\mu))\Big|\\
         &\leq 15/4\zeta\leq 4\zeta
     \end{align*}
     where the last inequalities follow by Lemma~\ref{Lem:sens-single-node}.
 \end{proof}

 Let us denote Algorithm~\ref{alg:DP-TopDown-DT} by $M$. Consider a fix decision tree $T$. We prove that, given $\mc{S}\sim\mc{S}'$ and $\mu,\mu'$, Algorithm~\ref{alg:DP-TopDown-DT} chooses the same leaf and split function with high probability. 

 Let $\mc{C}= leaves(T)\times F$ denote the set of possible split candidates. For each $(\ell,h)\in \mc{C}$, $\mathrm{im}_{\ell,h,\mu}(\mc{S})$ denotes the improvement gained in classification of dataset $\mc{S}$ by splitting $T$ at leaf $\ell$ according to split function $h$. Similarly, we have $\mr{im}_{\ell,h,\mu'}(\mc{S}')$. Provided that $\mathtt{d}(\mu,\mu')\leq \zeta$, by Lemma~\ref{lem:im-sensitivity}, the robust sensitivity of quality score $\mr{im}_{\ell,h,\mu}$ is at most $4\zeta$. Similar to the proof of Theorem~\ref{thm:weighted-EM} it follows that 
     \begin{align*}
         \frac{\Pr[M(\mc{S},\mu) = (\ell,h)]}{\Pr[M(\mc{S}',\mu') = (\ell,h)]} \leq \exp{(8\cdot\eta\cdot\zeta)}.
     \end{align*}

    This means each selection procedure where \texttt{DP-TopDown} selects a leaf and a splitting function is $(8\cdot\eta\cdot\zeta, 0, \zeta)$-differentially private. Using composition theorem for differentially private mechanisms, Theorem~\ref{thm:Sequential-Composition}, yields privacy guarantee 
        \[
            \tilde{\epsilon}=8t\cdot\eta\cdot\zeta 
        \]
    for the construction of the internal nodes. We use $\tilde{\epsilon}$ for labeling the leaves using Laplace Mechanism. Since the leaves partition dataset, this preserves $\tilde{\epsilon}$-differential privacy by parallel composition of deferentially private mechanisms (Theorem~\ref{thm:parallel-composition}). Overall, \texttt{TopDown-DT} is an $(16t\cdot\eta\cdot\zeta,0,\zeta)$-differentially private weak learner.
    
    \begin{rem}
        Using advanced composition for differentially private mechanisms, Theorem~\ref{thm:Sequential-Composition}, for every $\tilde{\delta}>0$ yields privacy guarantee 
        \[
            \tilde{\epsilon}_{\tilde{\delta}}=t(8\cdot\eta\cdot\zeta)^2 + 8\cdot\eta\cdot\zeta\sqrt{t\log(1/\tilde{\delta})}
        \]
        for the construction of the internal nodes. We use $\tilde{\epsilon}_{\tilde{\delta}}$ for labeling the leaves using Laplace Mechanism. Since the leaves partition dataset, this preserves $\tilde{\epsilon}_{\tilde{\delta}}$-differential privacy by parallel composition of deferentially private mechanisms (Theorem~\ref{thm:parallel-composition}). Overall, \texttt{TopDown-DT} is an $(2\tilde{\epsilon}_{\tilde{\delta}},\tilde{\delta},\zeta)$-differentially private weak learner.
    \end{rem}

\section{Approximate Differential Privacy}
\label{sec:approximate-dp}
Figures \ref{fig:lock} and \ref{fig:best} compare the cross validation
average accuracy on Adult dataset in the pure and approximate
differential privacy regimes, for two different strategies of
hyperparemeter selection: oblivious to $\epsilon$
(Figure~\ref{fig:lock}), and $\epsilon$-dependent
(Figure~\ref{fig:best}). This emphasizes the importance of tuning
hyperparameters for each choice of $\epsilon$ \emph{separately}. For
approximate differential privacy, we consider the small constant value
of $\delta = 10^{-5}$, the same as that used by DP-Bag
\cite{DP-Bagging-JordonYS19}.

\begin{figure}[ht]
\begin{center}
\centerline{\includegraphics[width=0.49\textwidth]{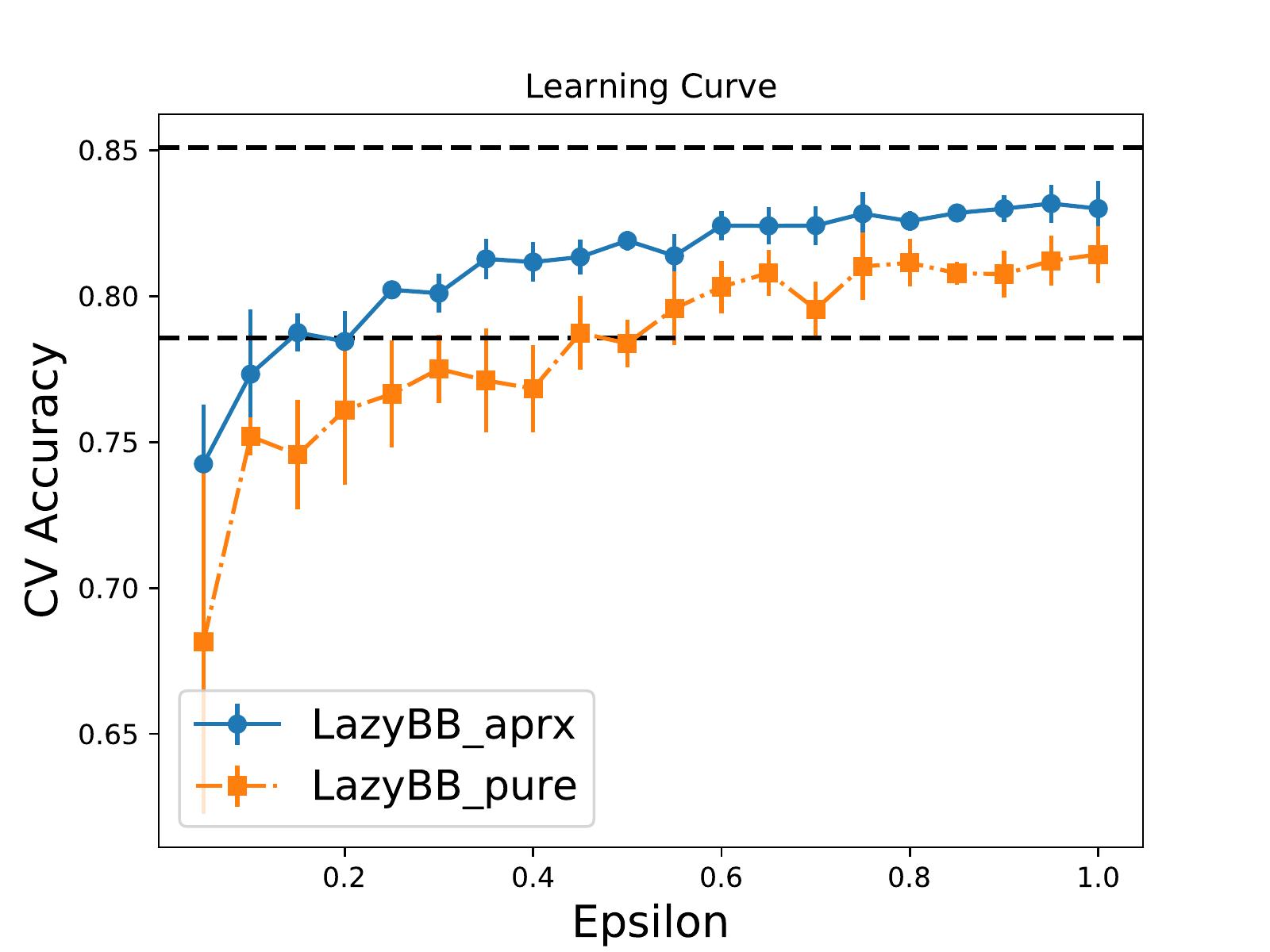}}
\caption{CV accuracy on Adult of $(\epsilon,\delta)$-DP
  \texttt{LazyBB} $(\kappa = 1/4, ~\lambda = 1/4, ~\tau = 99)$ with
  \texttt{DP-1R}, $\delta \in \{0, 10^{-5}\}$, varying $\epsilon$,
  vs. non-private baselines.}
\label{fig:lock}
\end{center}

\end{figure}

When we set hyperparameters identically for each $\epsilon$, using
approximate differential privacy can allow significantly increased
accuracy at each $\epsilon$. We found this to be the case especially
for higher $\tau$; we select $\tau = 99$ to illustrate. However, if we
are allowed to separately optimize for each $\epsilon$, the
significance of this advantage disappears. Though average accuracy
clearly improves, it is not outside one standard deviation of average
accuracy for pure differential privacy. It seems that boosted 1-Rules
are too simple to distinguish between pure and approximate
differential privacy constraints on this small dataset.

\begin{figure}[ht]

\begin{center}
\centerline{\includegraphics[width=0.49\textwidth]{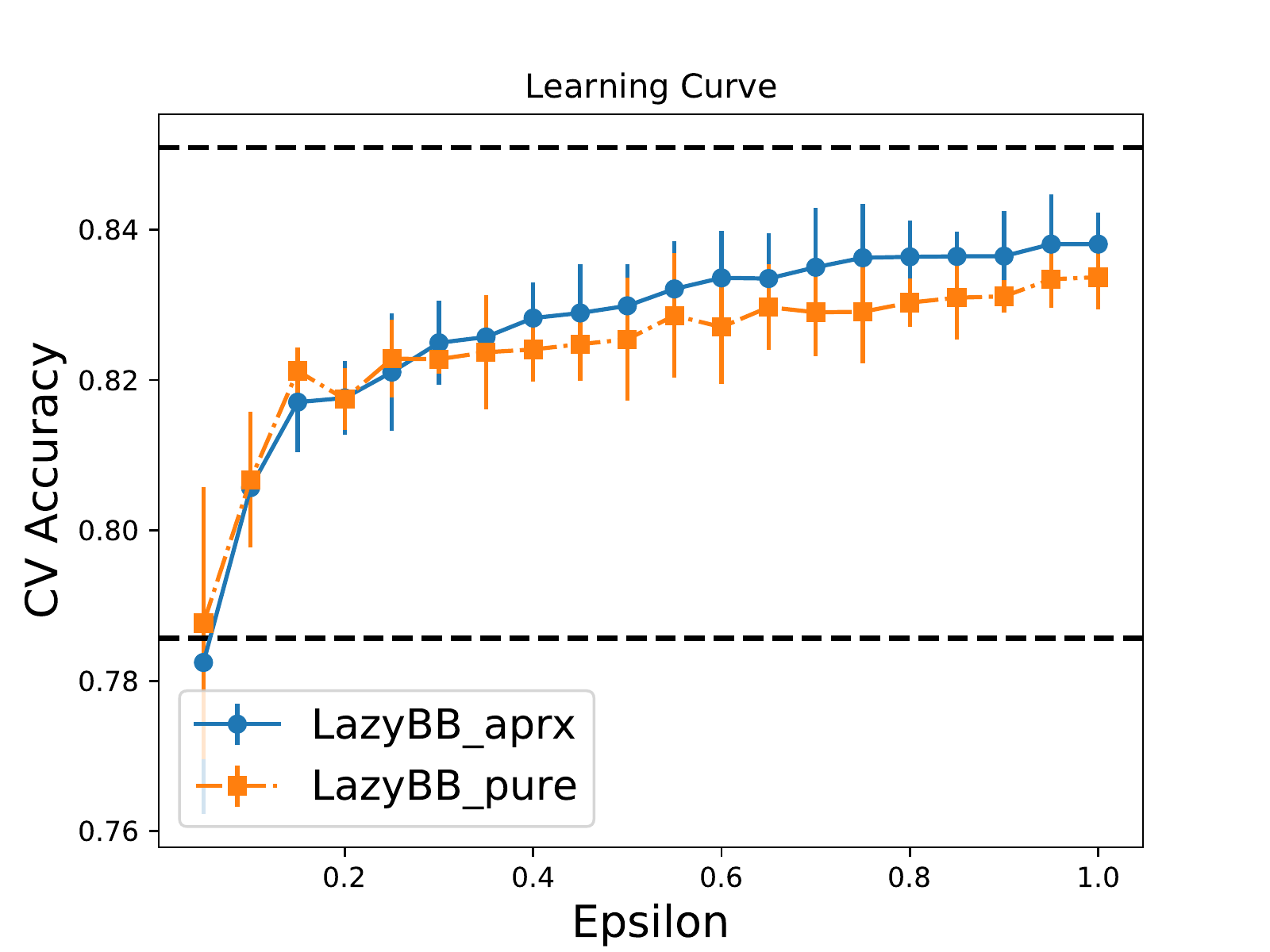}}
\caption{CV accuracy on Adult of $(\epsilon,\delta)$-DP
  \texttt{LazyBB} with \texttt{DP-1R}, $\delta \in \{0, 10^{-5}\}$,
  varying $\epsilon$, vs. non-private baselines, with best model for
  each $\epsilon$ displayed.}
\label{fig:best}
\end{center}

\end{figure}




\clearpage

\bibliography{sample}
\bibliographystyle{abbrvnat}

\newpage
\appendix

\section{Sparsity statistics of the experiments}
In Section~\ref{sec:experiments}, we discussed sparsity and interpretability of \texttt{LazyBB} with 1-Rules. Here we share the complete table of sparsity measurements for all the experiments. For each level of privacy, we use the hyper-parameter selected by cross-validation and repeated the experiment 5 times to obtain confidence bounds.
\begin{table}[ht]
    \centering
    \begin{tabular}{ cccc } 
     \toprule
     $\epsilon$ & features count mean & features count std &  \% features\\  \midrule
    0.05 & 4.6  & 0.489 &  2.83\% \\ 
    0.10 & 4.8  & 0.400 &  2.96\% \\ 
    0.15 &  3.8 & 0.400 &  2.34\% \\ 
    0.20 &  3.6 & 1.019 &  2.22\% \\ 
    0.25 &  7.0 & 0.632 &  4.32\% \\ 
    0.30 & 13.8 & 1.166 &  8.51\% \\ 
    0.35 &  7.6 & 0.489 &  4.69\% \\ 
    0.40 &  6.4 & 0.800 &  3.95\% \\ 
    0.45 & 19.2 & 2.785 &  11.85\% \\ 
    0.50 & 12.8 & 0.400 &  7.90\% \\ 
    1.00 & 30.6 & 1.200 &  18.88\% \\ 
    3.00 & 72.8 & 2.481 &  44.93\% \\ 
    5.00 & 49.8 & 2.785 &  30.74\% \\ \bottomrule
    \end{tabular}
    \caption{Sparsity measurements for Adult dataset.}
    \label{tab:sparsity_adult_full}
\end{table}
\begin{table}
    \begin{tabular}{ cccc } 
     \toprule
     $\epsilon$ & features count mean & features count std &  \% features\\  \midrule
0.05 &  6.0 &  0.632 &   7.50\%      \\
0.10 &  12.4 &  1.744 &   15.50\%      \\
0.15 &  19.6 &  1.625 &   24.50\%      \\
0.20 &  16.8 &  1.327 &   21.00\%      \\
0.25 &  11.2 &  0.748 &   14.00\%      \\
0.30 &  10.2 &  0.748 &   12.75\%      \\
0.35 &  26.4 &  1.356 &   33.00\%      \\
0.40 &  19.8 &  2.482 &   24.75\%      \\
0.45 &  34.4 &  1.497 &   43.00\%      \\
0.50 &  25.4 &  2.653 &   31.75\%      \\
1.00 &  54.2 &  3.187 &   67.75\%      \\
3.00 &  44.2 &  2.227 &   55.25\%      \\
5.00 &  32.0 &  2.098 &   40.00\%      \\

    \bottomrule
    \end{tabular}
    \caption{Sparsity measurements for Cod-RNA dataset.}
    \label{tab:sparsity_codrna_full}
\end{table}
\begin{table}
    \begin{tabular}{ cccc } 
     \toprule
     $\epsilon$ & features count mean & features count std &  \% features\\  \midrule
0.05 &  4.6 &  0.490 &   3.93\%      \\
0.10 &  7.2 &  1.166 &   6.15\%      \\
0.15 &  5.8 &  0.748 &   4.95\%      \\
0.20 &  8.6 &  1.497 &   7.35\%      \\
0.25 &  6.2 &  0.748 &   5.29\%      \\
0.30 &  5.6 &  0.490 &   4.78\%      \\
0.35 &  9.0 &  0.894 &   7.69\%      \\
0.40 &  9.8 &  1.166 &   8.37\%      \\
0.45 &  9.4 &  1.356 &   8.03\%      \\
0.50 &  11.8 &  1.720 &   10.08\%      \\
1.00 &  14.4 &  1.625 &   12.03\%      \\
3.00 &  28.8 &  2.926 &   24.61\%      \\
5.00 &  11.8 &  0.748 &   10.08\%      \\

    \bottomrule
    \end{tabular}
    \caption{Sparsity measurements for Mushroom dataset.}
    \label{tab:sparsity_mushroom_full}
\end{table}

\clearpage
\section{Hyperparameters}
These are the hyperparemeters selected by cross-validation of boosted
1-Rules over each of our datasets. The privacy vs. accuracy curves use
these settings for each value of $\epsilon$.

\begin{table}[ht]
    \centering
    \begin{tabular}{ cccc } 
     \toprule
     $\epsilon$ & density & learning rate &  no. estimators\\  \midrule
    0.05 &   0.50 &   0.50   & 5 \\
    0.10 &   0.45 &   0.50   & 5 \\
    0.15 &   0.50 &   0.40   & 5 \\
    0.20 &   0.50 &   0.30   & 5 \\
    0.25 &   0.35 &   0.50   & 9 \\
    0.30 &   0.40 &   0.40   & 19 \\
    0.35 &   0.30 &   0.45   & 9 \\
    0.40 &   0.35 &   0.50   & 9 \\
    0.45 &   0.40 &   0.45   & 25 \\
    0.50 &   0.35 &   0.50   & 15 \\
    1.00 &   0.35 &   0.45   & 39 \\
    3.00 &   0.35 &   0.45   & 99 \\
    5.00 &   0.35 &   0.45   & 75 \\
    \bottomrule
    \end{tabular}
    \caption{Hyperparameters selected by cross-validation for Adult dataset.}
 \end{table}
 \begin{table}[ht]

    \begin{tabular}{ cccc } 
     \toprule
     $\epsilon$ & density & learning rate &  no. estimators\\  \midrule
    0.05 &   0.50 &   0.50   & 9 \\
    0.10 &   0.50 &   0.35   & 19 \\
    0.15 &   0.50 &   0.50   & 29 \\
    0.20 &   0.40 &   0.50   & 25 \\
    0.25 &   0.50 &   0.45   & 25 \\
    0.30 &   0.50 &   0.45   & 25 \\
    0.35 &   0.45 &   0.35   & 49 \\
    0.40 &   0.45 &   0.45   & 39 \\
    0.45 &   0.50 &   0.40   & 65 \\
    0.50 &   0.45 &   0.50   & 49 \\
    1.00 &   0.40 &   0.50   & 99 \\
    3.00 &   0.30 &   0.40   & 99 \\
    5.00 &   0.35 &   0.40   & 99 \\
    \bottomrule
    \end{tabular}
    \caption{Hyperparameters selected by cross-validation for Cod-Rna dataset.}
  \end{table}
  \begin{table}[ht]
    
    \begin{tabular}{ cccc } 
     \toprule
     $\epsilon$ & density & learning rate &  no. estimators\\  \midrule
    0.05 &   0.45 &   0.50   & 5 \\
    0.10 &   0.50 &   0.40   & 9 \\
    0.15 &   0.50 &   0.45   & 9 \\
    0.20 &   0.50 &   0.40   & 15 \\
    0.25 &   0.30 &   0.40   & 9 \\
    0.30 &   0.35 &   0.50   & 9 \\
    0.35 &   0.40 &   0.35   & 15 \\
    0.40 &   0.45 &   0.40   & 19 \\
    0.45 &   0.35 &   0.20   & 19 \\
    0.50 &   0.45 &   0.25   & 25 \\
    1.00 &   0.25 &   0.30   & 29 \\
    3.00 &   0.20 &   0.20   & 75 \\
    5.00 &   0.20 &   0.50   & 29 \\
    \bottomrule
    \end{tabular}
    \caption{Hyperparameters selected by cross-validation for Mushroom dataset.}
    \label{tab:hyperparameters_codrna}

\end{table}

\clearpage
\section{Gap Between Theory and Experiments for Test Error}
As discussed in section~\ref{sec:experiments}, there is a large gap between lower bounds predicted by large margin theory and Rademacher complexity, and the actual performance. The following table compares the best guaranteed lower bound derived by estimated Rademacher complexity and the test accuracy. The test accuracy of Adult dataset is obtained by evaluating the model on the test set, which was not touched during training. For Cod-RNA and Mushroom dataset there is no canonical test set available, so we report cross-validation accuracy. 

\begin{table}[ht]
    \centering
    \begin{tabular}{ cccc } 
     \toprule
    Dataset & Rademacher Estimate of Test Accuracy & (CV) test accuracy\\  \midrule
    Adult &   0.37 &   0.83 \\
    Cod-Rna & 0.09 & 0.86\\
    Mushroom & 0.49 & 0.98\\
    \bottomrule
    \end{tabular}
    \caption{Comparison between Rademacher estimates of generalization performance and experimental generalization performance for boosted 1-Rules, at $\epsilon = 1$.}
    \label{tab:hyperparameters_adult}
\end{table}

\end{document}